\documentclass[11 pt]{article}
\usepackage{style}

\title{\LARGE\bfseries Natural Policy Gradient as Doubly Smoothed Policy Iteration: A Bellman-Operator Framework}
\author{
Phalguni Nanda$^1$ and Zaiwei Chen$^2$\\
{\small \textit{
Edwardson School of Industrial Engineering, Purdue University}}\\
{\small $^1$\href{mailto:nanda14@purdue.edu}{\texttt{nanda14@purdue.edu}}; $^2$\href{mailto:chen5252@purdue.edu}{\texttt{chen5252@purdue.edu}} 
}
}
\date{}
\begin{document}

\maketitle
\begin{abstract}
In this work, we show that natural policy gradient, a core algorithm in reinforcement learning, admits an exact formulation as a smoothed and averaged form of policy iteration. Specifically, we introduce \emph{doubly smoothed policy iteration} (DSPI), a Bellman-operator framework in which each policy is obtained by applying a regularized greedy step to a weighted average of past $Q$-functions. DSPI includes policy iteration, dual-averaged policy iteration, natural policy gradient, and more general policy dual averaging methods as special cases.

Using only monotonicity and contraction of smoothed Bellman operators, we prove distribution-free global geometric convergence of DSPI. Consequently, standard natural policy gradient and policy dual averaging achieve an iteration complexity of $\mathcal{O}((1-\gamma)^{-1}\log((1-\gamma)^{-1}\epsilon^{-1}))$
for computing an $\epsilon$-optimal policy, without modifying the MDP, adding regularization beyond the mirror map inherent in the update, or using adaptive, trajectory-dependent stepsizes. For the unregularized greedy case, corresponding to dual-averaged policy iteration, we also prove finite termination. The same Bellman-operator framework further extends to discounted MDPs with linear function approximation and stochastic shortest path problems.
\end{abstract}

\section{Introduction}\label{sec:intro}
In the last decade, reinforcement learning (RL) \citep{sutton2018reinforcement} has emerged as a principled framework for large-scale sequential decision making, with applications ranging from autonomous vehicles \citep{yurtsever2020survey} and robotics \citep{deeprl_robotics} to large language models \citep{brown2020language}. Mathematically, RL problems are typically modeled as Markov decision processes (MDPs) \citep{puterman2014markov}. The seminal work of \cite{bellman1957dynamic} showed that solving an MDP reduces to finding a fixed point of the Bellman equation. By leveraging key properties of the Bellman operator, most notably contraction and monotonicity, a range of principled algorithms have been developed, including value iteration (VI) \cite{bellman1957dynamic}, policy iteration (PI) \cite{howard1964dynamic}, and their variants \citep{puterman2014markov,bertsekas1995dynamic,bertsekas1996neuro}.

In the modern era of sequential decision making, it was recognized that an MDP can be formulated as a continuous optimization problem over the policy space. From this perspective, the policy gradient (PG) method was introduced \cite{sutton1999policy}. Building on this viewpoint and further exploiting the geometry of the policy space, the natural policy gradient (NPG) method, along with more general approaches such as policy mirror ascent (PMA) and policy dual averaging (PDA), was developed \citep{kakade2001natural,lan2023policy,ju2022policy}. Over time, practical variants of these gradient-based methods, such as trust region policy optimization \citep{schulman2015trust} and proximal policy optimization \citep{schulman2017proximal}, have become dominant approaches in large-scale applications \citep{deeprl_robotics,brown2020language}.

The empirical success of NPG and its variants has motivated extensive theoretical work on their convergence properties. Existing approaches for studying NPG can be broadly categorized into two groups. One prominent line of work adopts a gradient-based optimization viewpoint, leveraging tools such as the policy gradient theorem, the performance difference lemma, the Polyak \L{}ojasiewicz condition \cite{sutton1999policy,agarwal2021theory}, and techniques from continuous optimization \citep{lan2023policy,xiao2022convergence}. A second line of work views NPG as an approximation of PI. Indeed, the original NPG paper \cite[Theorem 2]{kakade2001natural} observed that, as the stepsize approaches infinity, NPG recovers PI. Following this perspective, NPG has been analyzed as an approximation of PI by leveraging properties of the Bellman operator and separately bounding the approximation error \cite{chen2025API,bhandari2021linear}.

Despite substantial progress, existing analyses still leave open a conceptual gap. The defining structural property of discounted MDPs is the $\ell_\infty$-norm contraction of the Bellman operator, which underlies the design and analysis of classical methods such as VI and PI. This property is not naturally visible when NPG is analyzed as a continuous optimization method over the policy space. Consequently, many optimization-based analyses either do not establish global geometric convergence \cite{agarwal2021theory} or achieve it by introducing additional regularization into the MDP or the algorithm to induce curvature \cite{cen2022fast,lan2023policy,shani2020adaptive}. In contrast, approaches that view NPG as an approximation of PI \cite{bhandari2018finite,chen2025API} exploit Bellman operator properties and obtain distribution-free geometric convergence without additional regularization, but require stepsizes chosen adaptively based on the algorithm trajectory to control the approximation error. See Section~\ref{subsec:literature} for further details.

More fundamentally, a unified perspective connecting classical dynamic programming methods based on Bellman iteration, such as PI, with modern gradient-based methods, such as NPG and PDA, remains lacking at both the algorithmic and analytical levels. In this paper, we address this gap by introducing \textit{doubly smoothed policy iteration} (DSPI), a unified framework that includes PI, dual-averaged PI, NPG, and PDA as special cases. This perspective enables a unified analysis that yields strong guarantees for a broad class of algorithms. We summarize our contributions below.

\begin{itemize}
    \item \textbf{Doubly Smoothed Policy Iteration.} We introduce DSPI, in which the policy is updated via a regularized greedy step with respect to a weighted average of past $Q$-functions. This structure admits a compact Bellman operator formulation (cf. Algorithm~\ref{alg:DSPI}, Line~4). We show that DSPI captures a broad class of existing algorithms, including PI, dual-averaged PI, and, most notably, NPG and PDA. Crucially, for NPG and PDA, interpreting these gradient-based methods as smoothed and averaged variants of PI allows the updates to be expressed via the smoothed Bellman operator, thereby enabling a principled analysis based on its properties.

    \item \textbf{A Unified Convergence Analysis.} We show that DSPI converges asymptotically whenever the stepsizes are non-summable. Moreover, under a constant stepsize, DSPI achieves distribution-free global geometric convergence. In particular,
\begin{align*}
    \|V^{\pi_k}-V^*\|_\infty 
    \leq (1-(1-\gamma)\beta)^{k-1}
    \left(\gamma\|V^*-V^{\pi_0}\|_\infty + 1\right), \quad \forall\,k \geq 1,
\end{align*}
where $V^{\pi_0}$, $V^{\pi_k}$, and $V^*$ denote the value functions of the initial policy, the $k$-th iterate, and an optimal policy, respectively, $\gamma$ is the discount factor, and $\beta$ is the constant stepsize. See Theorem~\ref{thm:DSPI} for details. As byproducts,
\begin{itemize}
    \item this result implies an iteration complexity of $\mathcal{O}((1-\gamma)^{-1}\log((1-\gamma)^{-1}\epsilon^{-1}))$ for NPG and PDA to achieve $\|V^{\pi_k}-V^*\|_\infty \leq \epsilon$ (cf. Theorem~\ref{thm:NPG}), matching the state of the art without requiring additional regularization of the MDP or the algorithm, or the use of adaptive stepsizes;
    \item it also enables us to establish finite termination of dual-averaged PI; in particular, an iteration complexity of $\mathcal{O}(mn(1-\gamma)^{-1}\log((1-\gamma)^{-1}))$ for finding an optimal policy (cf. Theorem~\ref{thm:strong_polynomial}), where $n$ and $m$ denote the sizes of the state and action spaces, respectively.
\end{itemize}
Moreover, the algorithm and analysis extend to settings such as discounted MDPs with linear function approximation, where we obtain the same convergence rate up to a function approximation error (cf. Appendix~\ref{sec:NPG-LFA}), and stochastic shortest path problems, where we establish analogous geometric convergence (cf. Appendix~\ref{sec:SSP}).
\end{itemize}

\subsection{Related Literature}\label{subsec:literature}
Since we focus on the MDP setting rather than the model-free RL setting, we restrict our discussion below to PI and NPG, rather than results that broadly analyze actor--critic algorithms. 

\textbf{Policy Iteration.} The PI method was first proposed in \cite{howard1964dynamic}. The global geometric convergence of PI was shown in \cite{howard1964dynamic,scherrer2016improved,puterman2014markov}. The proof has two main steps: (1) showing that the performance of the policies is monotonically improving, and (2) leveraging the contraction property of the Bellman operator to establish a contractive recursion. This approach also inspires our analysis of DSPI. Beyond global geometric convergence, it was shown that PI can terminate in a finite number of iterations \cite{scherrer2016improved,ye2011simplex}. The key property that enabled such strong results is that PI gradually eliminates sub-optimal actions. Such a property also plays a key role in showing the finite termination of dual-averaged PI in this work. Other variants of PI, including simplex PI, modified PI, and incremental PI, are analyzed in \cite{scherrer2016improved,ye2011simplex,bertsekas1996neuro,bertsekas1995dynamic}.

\textbf{Policy Gradient.} The PG method was first proposed in \cite{sutton1999policy}, which treats a weighted sum of the value function as a scalar objective and performs projected gradient ascent in the policy space. A comprehensive analysis of PG was carried out in \cite{agarwal2021theory} and subsequent works. Beyond the setting of finite state-action space MDPs, PG methods for more general MDPs (e.g., continuous state-action spaces with potentially unbounded costs) have also been studied in \cite{gupta2026operator}. Since the focus of this work is on NPG and more general PDA methods, we do not delve further into the literature on the convergence-rate analysis of vanilla PG.

\textbf{Natural Policy Gradient.}
NPG, proposed in \cite{kakade2001natural}, can be viewed as PG with a preconditioner. Alternatively, it can be interpreted as PMA \cite{lan2023policy,xiao2022convergence} or PDA \cite{ju2022policy}, with the divergence-generating function chosen as the Shannon entropy. The majority of existing analyses of NPG suffer from at least one of the following limitations: (1) lack of global geometric convergence \cite{agarwal2021theory,shani2020adaptive}; (2) dependence on the initial distribution or a concentrability coefficient \cite{bhandari2021linear,xiao2022convergence,lan2023policy,alfano2023novel}; (3) additional regularization of the MDP or the algorithm \cite{shani2020adaptive,cen2022fast,lan2023policy,li2024homotopic,zhan2023policy,xiao2022convergence}; or (4) reliance on adaptive stepsizes that depend on the policy trajectory \cite{bhandari2021linear,chen2025API}. See Table \ref{table:1} from Appendix \ref{app:comparison} for more details. The only exception is \cite{ju2026strongly}, which provides the state-of-the-art analysis of NPG from a continuous optimization perspective\footnote{The results in \cite{ju2026strongly} were developed for general PMA, which covers NPG as a special case.}, but their choice of stepsize depends on the Bellman error (also called the advantage gap) of the initial policy, which is not required in our result.

\section{Background}
Consider an infinite horizon discounted MDP defined by the tuple $\mathcal{M} = (\mathcal{S}, \mathcal{A}, p, \mathcal{R}, \gamma)$, where $\mathcal{S}$ and $\mathcal{A}$ denote the finite state and action spaces, respectively, $\{p(\cdot \mid s, a)\}_{(s,a)\in\mathcal{S}\times\mathcal{A}}$ denotes the transition kernel, $\mathcal{R}:\mathcal{S}\times\mathcal{A}\to[0,1]$ is the reward function, and $\gamma \in (0,1)$ is the discount factor. Throughout, we denote $n = |\mathcal{S}|$ and $m = |\mathcal{A}|$. 

Given a policy $\pi:\mathcal{S}\to\Delta(\mathcal{A})$, where $\Delta(\mathcal{A})$ denotes the set of probability distributions supported on $\mathcal{A}$, the $Q$-function $Q^\pi:\mathcal{S}\times\mathcal{A}\to\mathbb{R}$ is defined as 
\begin{align*}
    Q^\pi(s,a)=\mathbb{E}_\pi\left[\sum_{k=0}^\infty \gamma^k \mathcal{R}(S_k, A_k) \,\middle|\, S_0 = s, A_0 = a \right],\quad \forall\,(s,a)\in\mathcal{S}\times \mathcal{A},
\end{align*}
where the expectation $\mathbb{E}_\pi[\ \cdot \ ]$ is taken with respect to the randomness in both the action selection $A_k \sim \pi(\cdot \mid S_k)$ and the state transitions $S_{k+1} \sim p(\cdot \mid S_k, A_k)$. With $Q^\pi$ defined above, the value function $V^\pi:\mathcal{S}\to\mathbb{R}$ is  defined as $V^\pi(s)=\sum_{a\in\mathcal{A}}\pi(a|s)Q^\pi(s,a)$ for all $s\in\mathcal{S}$. Since the MDP has finitely many states and actions, $Q^\pi$ and $V^\pi$ can be equivalently viewed as vectors in $\mathbb{R}^{mn}$ and $\mathbb{R}^n$, respectively.

A policy $\pi^*$ is \emph{optimal} if $V^* := V^{\pi^*} \geq V^\pi$ for all $\pi$, or equivalently, $Q^* := Q^{\pi^*} \geq Q^\pi$ for all $\pi$. Throughout this paper, inequalities between vectors are understood componentwise. For notational simplicity, for any vector $x \in \mathbb{R}^{mn}$ (e.g., representing a policy or a $Q$-function), we denote by $x(s) \in \mathbb{R}^m$ the vector whose $a$-th entry is $x(s,a)$.

\subsection{The Bellman Equation and Policy Iteration}\label{subsec:Bellman}
The key to solving an MDP lies in the Bellman equation. We next present the Bellman equations for policy evaluation and policy optimization. To facilitate the connection between PI and NPG developed later in this paper, we express these equations in terms of $Q$-functions.

Given a policy $\pi$, its $Q$-function $Q^\pi$ is the unique solution to the Bellman equation $Q = \mathcal{H}^\pi(Q)$, where $\mathcal{H}^\pi:\mathbb{R}^{mn}\to\mathbb{R}^{mn}$ is the Bellman operator associated with policy $\pi$, defined as 
\begin{align*}
    [\mathcal{H}^\pi(Q)](s,a) = \mathcal{R}(s,a) + \gamma \sum_{s'\in\mathcal{S}} p(s'\mid s,a)
    \sum_{a'\in\mathcal{A}} \pi(a'\mid s') Q(s',a'),\quad \forall\,(s,a)\in\mathcal{S}\times\mathcal{A}.
\end{align*}
Similarly, the optimal $Q$-function $Q^*$ is the unique solution to the Bellman optimality equation $Q = \mathcal{H}(Q)$, where $\mathcal{H}:\mathbb{R}^{mn}\to\mathbb{R}^{mn}$ is the Bellman optimality operator defined as 
\begin{align*}
    [\mathcal{H}(Q)](s,a) = \mathcal{R}(s,a) + \gamma \sum_{s'\in\mathcal{S}} p(s'\mid s,a)
    \max_{a'\in\mathcal{A}} Q(s',a'),\quad \forall\,(s,a)\in\mathcal{S}\times\mathcal{A}.
\end{align*}
Moreover, once $Q^*$ is obtained, an optimal policy $\pi^*$ can be computed by choosing actions greedily with respect to $Q^*$. The operators $\mathcal{H}^\pi$ and $\mathcal{H}$ enjoy several useful properties. In particular, both are contraction mappings with respect to the $\ell_\infty$ norm with contraction factor $\gamma$, and are monotone in the sense that $F(Q_1)\le F(Q_2)$ whenever $Q_1\le Q_2$, where $F$ denotes either $\mathcal{H}^\pi$ or $\mathcal{H}$ \citep{puterman2014markov,bertsekas1996neuro}. One should not confuse the definition of monotonicity here with the notion of monotonicity in variational inequalities \cite{ryu2022large}. These two properties enable the design and analysis of several classical algorithms, such as VI and PI.

When expressed in terms of $Q$-functions, the PI algorithm can be compactly written as
\begin{align}\label{algo:PI:Bellman}
    \mathcal{H}^{\pi_{k+1}}(Q^{\pi_k})=\mathcal{H}(Q^{\pi_k}).
\end{align}
Note that \eqref{algo:PI:Bellman} does not define a unique update. In particular, $\pi_{k+1}$ can be any policy such that, when its associated Bellman operator $\mathcal{H}^{\pi_{k+1}}$ is applied to the previous $Q$-function $Q^{\pi_k}$, the result coincides with applying the Bellman optimality operator to $Q^{\pi_k}$. In view of the definition of the Bellman operators, one concrete way to implement the PI update in \eqref{algo:PI:Bellman} is
\begin{align}\label{eq:PI:explicit}
    \pi_{k+1}(s)\in\argmax_{\mu\in\Delta(\mathcal{A})}\{\mu^\top Q^{\pi_k}(s)\},\quad \forall\,s\in\mathcal{S},
\end{align}
where, according to our notation, $\pi_{k+1}(s)$ (respectively, $Q^{\pi_k}(s)$) denotes the $m$-dimensional vector whose $a$-th entry is $\pi_{k+1}(a\mid s)$ (respectively, $Q^{\pi_k}(s,a)$).

The global geometric convergence of PI is well established in the literature \cite{puterman2014markov}. The analysis proceeds in two main steps: (1) using the monotonicity of the Bellman operators to show that the policies are monotonically improving, i.e., $Q^{\pi_k} \leq Q^{\pi_{k+1}}$; (2) using the contraction property of the Bellman operators to show that the convergence error $\|Q^*-Q^{\pi_k}\|_\infty$ admits a contractive recursion. For completeness, we provide an analysis of PI in Appendix~\ref{ap:PI_analysis}.

\subsection{Natural Policy Gradient and Policy Dual Averaging}
Since an optimal policy maximizes the value function $V^\pi(s)$ uniformly over all states $s$, solving an MDP is equivalent to solving an optimization problem over the policy space with the scalar objective
$V^\pi(\rho) := \sum_{s\in\mathcal{S}} \rho(s) V^\pi(s)$,
where $\rho \in \Delta(\mathcal{S})$ can be interpreted as the initial state distribution. This viewpoint naturally motivates gradient-based algorithms. In particular, the vanilla PG method performs projected gradient ascent over the policy space \citep{sutton1999policy}, while the NPG method performs gradient ascent under a geometry induced by the Fisher information matrix \citep{kakade2001natural}. 

\begin{algorithm}[ht]\caption{Natural Policy Gradient}\label{alg:PDA}
	\begin{algorithmic}[1]
		\STATE \textbf{Input:} Initialize $\pi_0$ as the uniform policy and $\theta_0=\mathbf{0}\in\mathbb{R}^{mn}$.
		\FOR{$k=0,1,2,3,\cdots$}
        \STATE $\theta_{k+1}=\theta_k+\alpha_k Q^{\pi_k}$   
        \STATE $\pi_{k+1}(s)=\argmax_{\mu\in\Delta(\mathcal{A})}
    \left\{\mu^\top \theta_{k+1}(s)+h(\mu)\right\}$ for any $s\in\mathcal{S}$.    
		\ENDFOR
	\end{algorithmic}
\end{algorithm} 

There are many equivalent formulations of the NPG update. In this paper, to facilitate its connection with PI, we adopt the NPG update rule from \cite{agarwal2021theory}\footnote{It corresponds to Q-NPG in \cite[Page 25]{agarwal2021theory} when using a tabular representation.}, presented in Algorithm~\ref{alg:PDA}, where $h : \Delta(\mathcal{A}) \to \mathbb{R}$ is the entropy function defined by $h(\mu) = -\sum_{a \in \mathcal{A}} \mu(a)\log \mu(a)$. Note that NPG is a special case of PDA \citep{ju2022policy}. Specifically, replacing $h(\cdot)$ in Algorithm~\ref{alg:PDA}, Line~4, with the negative of any valid divergence-generating function yields PDA.

\section{Doubly Smoothed Policy Iteration}\label{sec:DSPI}
To present DSPI, we begin by introducing the concept of smoothed Bellman operators. Let $\nu:\Delta(\mathcal{A})\to \mathbb{R}$ be a bounded, non-negative, and concave function; examples include the Shannon entropy, the Tsallis entropy, the shifted negative squared norm, and the zero function. Given a scalar $\eta\geq 0$, let the smoothed Bellman optimality operator $\mathcal{H}_{\eta}:\mathbb{R}^{mn}\to \mathbb{R}^{mn}$ be defined as 
\begin{align*}
    [\mathcal{H}_{\eta}(Q)](s,a) = \mathcal{R}(s,a)+
    \gamma \sum_{s'\in\mathcal{S}} p(s'\mid s,a) \max_{\mu\in\Delta(\mathcal{A})} (\mu^\top Q(s')+\eta \nu(\mu)),\quad\forall\,(s,a)\in\mathcal{S}\times\mathcal{A}.
\end{align*}
Given a policy $\pi$ and a scalar $\eta\geq 0$, let $\mathcal{H}_\eta^\pi:\mathbb{R}^{mn}\to \mathbb{R}^{mn}$ be defined as 
\begin{align*}
    [\mathcal{H}_\eta^\pi(Q)](s,a) = \mathcal{R}(s,a) + \gamma \sum_{s'\in\mathcal{S}} p(s'\mid s,a)
    (\pi(s')^\top Q(s')+\eta \nu(\pi(s'))),\quad\forall\,(s,a)\in\mathcal{S}\times\mathcal{A}.
\end{align*}
Note that $\mathcal{H}_{\eta}$ (respectively, $\mathcal{H}_\eta^\pi$) reduces to $\mathcal{H}$ (respectively, $\mathcal{H}^\pi$) when $\eta=0$. Moreover, the smoothed Bellman operators are also contraction mappings and are monotonic (cf. Appendix \ref{ap:smooth_Bellman_properties}).

With the smoothed Bellman operators in place, DSPI is presented in Algorithm \ref{alg:DSPI}. 
\begin{algorithm}[ht]\caption{Doubly Smoothed Policy Iteration}\label{alg:DSPI}
	\begin{algorithmic}[1]
		\STATE \textbf{Input:} Initialize $\pi_0(s)\in \argmax_{\mu\in\Delta(\mathcal{A})}\nu(\mu)$ for all $s\in\mathcal{S}$ and $\bar{Q}_0=\mathbf{0}$.
		\FOR{$k=0,1,2,3,\cdots$}
        \STATE Maintain a running average of past $Q$-functions: 
        \begin{align*}
            \bar{Q}_{k+1}=(1-\beta_k)\bar{Q}_k+\beta_k Q^{\pi_k},
        \end{align*}
        where $\beta_k\in [0,1]$ is the stepsize.
        \STATE Choose $\pi_{k+1}$ such that  
        \begin{align*}
            \mathcal{H}_{\eta_k}^{\pi_{k+1}}(\bar{Q}_{k+1})=\mathcal{H}_{\eta_k}(\bar{Q}_{k+1}),
        \end{align*}
        where $\eta_k=\tau \prod_{j=1}^k(1-\beta_j)$ with $\tau\geq 0$ being a tunable parameter.
		\ENDFOR
	\end{algorithmic}
\end{algorithm}
Similar to PI, Line 4 of Algorithm \ref{alg:DSPI} in general does not define a unique update. As long as $\pi_{k+1}$ is chosen such that, when its associated smoothed Bellman operator $\mathcal{H}_{\eta_k}^{\pi_{k+1}}$ is applied to the running average $\bar{Q}_{k+1}$, the result coincides with that of applying the smoothed Bellman optimality operator $\mathcal{H}_{\eta_k}$ to $\bar{Q}_{k+1}$, it is a valid update. A concrete way to implement this step is given by
\begin{align}\label{eq:DSPI_explicit}
        \pi_{k+1}(s)\in \argmax_{\mu \in \Delta(\mathcal{A})}
    \left\{
        \mu^\top \bar{Q}_{k+1}(s)
        +\tau \prod_{j=1}^k(1-\beta_j)\nu(\mu)
    \right\},\quad \forall\,s\in\mathcal{S}.
\end{align}
While \eqref{eq:DSPI_explicit} is more explicit, the Bellman operator viewpoint in Line~4 of Algorithm~\ref{alg:DSPI} more naturally captures the fundamental nature of the update.

Compared with PI presented in \eqref{algo:PI:Bellman} (or more explicitly, \eqref{eq:PI:explicit}), DSPI introduces two levels of \textit{smoothing}. First, it updates the policy using a weighted average of all past $Q$-functions rather than relying solely on the most recent one. Second, it replaces the greedy policy update with a softmax-based update, where $\nu(\cdot)$ acts as a regularization term. These two smoothing mechanisms motivate the name \emph{doubly smoothed policy iteration}. Note that the stepsize $\beta_k$ simultaneously controls the weights in computing $\bar{Q}_{k+1}$ and the degree of smoothing through the scaling $\eta_k$. As $\beta_k$ increases from zero to one, $\bar{Q}_{k+1}$ places more weight on the last iterate $Q^{\pi_k}$ and the effect of smoothing fades away. In the extreme case where $\beta_k\equiv 1$, DSPI reduces to PI.

\textbf{Natural Policy Gradient.}
In addition to capturing PI, we next show that DSPI serves as a unified algorithmic framework encompassing a broad class of algorithms, starting with NPG.

\begin{proposition}\label{prop:equivalence}
    In Algorithm~\ref{alg:DSPI}, when choosing $\beta_k = \alpha_k / \sum_{i=0}^k \alpha_i$, $\tau=1/\alpha_0$, and $\nu(\cdot)=h(\cdot)$, the sequence of policies generated by Algorithm~\ref{alg:DSPI} with Line 4 implemented via \eqref{eq:DSPI_explicit} is identical to that generated by Algorithm~\ref{alg:PDA}.
\end{proposition}
\begin{proof}[Proof of Proposition~\ref{prop:equivalence}]
Since both algorithms initialize at the uniform policy, it suffices to show that their update equations generate identical policies thereafter. By Line 3 of Algorithm \ref{alg:PDA}, we have
$\theta_{k+1} = \sum_{i=0}^{k} \alpha_i Q^{\pi_i}$ for all $k \ge 0$.
As a result, Line 4 of Algorithm \ref{alg:PDA} becomes
\begin{align*}
    \pi_{k+1}(s)=\,& \argmax_{\mu \in \Delta(\mathcal{A})}
    \left\{
        \mu^\top \left(\sum_{i=0}^{k} \alpha_i Q^{\pi_i}(s)\right)
        +  h(\mu)
    \right\}\\
    =\,& \argmax_{\mu \in \Delta(\mathcal{A})}
    \left\{
        \mu^\top \left(\frac{\sum_{i=0}^{k} \alpha_i Q^{\pi_i}(s)}{\sum_{i=0}^{k} \alpha_i}\right)
        + \frac{\alpha_0\tau}{\sum_{i=0}^{k} \alpha_i} h(\mu)
    \right\},\quad\forall\, s \in \mathcal{S},
\end{align*}
where the last equality follows from (i) the $\argmax$ operator being indifferent to constant scaling and (ii) $\tau=1/\alpha_0$. In view of the above equality and Line~4 of Algorithm~\ref{alg:DSPI} (which is equivalent to \eqref{eq:DSPI_explicit} since the Shannon entropy $h(\cdot)$ is strongly concave), it remains to show that
$\bar{Q}_{k+1} = \sum_{i=0}^{k} \alpha_i Q^{\pi_i} / \sum_{i=0}^{k} \alpha_i$
and
$\prod_{j=1}^k (1 - \beta_j) = \alpha_0 / \sum_{i=0}^k \alpha_i$
for all $k\geq 0$, which we prove by induction.

For $k=0$, since $\beta_0=1$, the base case clearly holds.
Suppose that the desired equalities hold for some $k\geq 0$. Then, by the update equation in Line~3 of Algorithm~\ref{alg:DSPI}, we have
\begin{align*}
    \bar{Q}_{k+2}
    =(1-\beta_{k+1})\bar{Q}_{k+1}+\beta_{k+1}Q^{\pi_{k+1}}
    =\frac{(\sum_{i=0}^{k}\alpha_i)\bar{Q}_{k+1}+\alpha_{k+1}Q^{\pi_{k+1}}}{\sum_{i=0}^{k+1}\alpha_i}
    =\frac{\sum_{i=0}^{k+1}\alpha_iQ^{\pi_i}}{\sum_{i=0}^{k+1}\alpha_i},
\end{align*}
where the last equality follows from the induction hypothesis. Similarly, we have
\begin{align*}
    \prod_{j=1}^{k+1} (1 - \beta_j)
    = \prod_{j=1}^{k} (1 - \beta_j) \cdot (1 - \beta_{k+1})
    = \frac{\alpha_0}{\sum_{i=0}^k \alpha_i}
      \left(1 - \frac{\alpha_{k+1}}{\sum_{i=0}^{k+1} \alpha_i}\right)
    = \frac{\alpha_0}{\sum_{i=0}^{k+1} \alpha_i}.
\end{align*}
The induction is complete.
\end{proof}

Proposition~\ref{prop:equivalence} is conceptually important, as it reveals that the widely used NPG algorithm, although originally developed from a gradient-based continuous optimization viewpoint, can be interpreted as a smoothed and averaged variant of classical PI. This perspective allows us to build on Bellman’s foundational insights, such as the contraction and monotonicity of the Bellman operator, to establish strong convergence guarantees for NPG via the analysis of DSPI.

\textbf{Policy Dual Averaging.} Since NPG is a special case of PDA when the divergence-generating function is chosen as the negative entropy \cite{ju2026strongly}, by setting $\nu(\cdot)$ to a shifted negative divergence-generating function, i.e., $\nu(\mu)=C-\omega(\mu)$ for a bounded convex divergence-generating function $\omega$, where $C:=\max_{\mu}\omega(\mu)$ ensures that $\nu(\mu)\geq 0$, DSPI also captures PDA. The proof follows identically to that of Proposition~\ref{prop:equivalence} and is therefore omitted.

\textbf{Dual-Averaged Policy Iteration.} DSPI is more general than PDA, as it allows choosing $\nu(\cdot)\equiv 0$. In this case, Algorithm~\ref{alg:DSPI} updates the policy greedily based on a weighted average of historical $Q$-functions. We refer to this update as dual-averaged PI. One should not confuse dual-averaged PI with incremental PI \cite{puterman2014markov}, where the averaging is performed in the policy space, whereas in dual-averaged PI the averaging is performed in the $Q$-function space.

\section{Analysis}\label{sec:analysis}
This section presents our main theoretical results. We begin with the convergence rates of DSPI in Section \ref{subsec:DSPI}. Due to its unified nature, this result immediately implies the distribution-free global geometric convergence of NPG and PDA in Section \ref{subsec:NPG}, and enables us to establish the finite termination of dual-averaged PI in Section \ref{subsec:DPI}.

\subsection{Convergence Rates of Doubly Smoothed Policy Iteration}\label{subsec:DSPI}
The following theorem presents the asymptotic convergence and convergence rates of the DSPI algorithm. The proof of Theorem~\ref{thm:DSPI} is presented in Section~\ref{sec:proof}. 

\begin{theorem}\label{thm:DSPI}
    Consider $\{\pi_k\}$ generated by Algorithm~\ref{alg:DSPI}. 
    \begin{enumerate}[(1)]
        \item When $\beta_k\in(0,1]$ and $\sum_{k=0}^\infty \beta_k = \infty$, we have $\lim_{k \to \infty} \|V^* - V^{\pi_k}\|_\infty = 0$.
        \item When choosing $\beta_0 = 1$, $\beta_k \equiv \beta \in (0,1]$ for all $k \geq 1$, we have
        \begin{align*}
            \|V^* - V^{\pi_k}\|_\infty
            \le (1 - (1 - \gamma)\beta)^{k-1}
            \left( \gamma \|V^* - V^{\pi_0}\|_\infty + \tau \nu_{\max} \right),
            \quad \forall\, k \geq 1,
        \end{align*}
        where $\nu_{\max} := \max_{\mu \in \Delta(\mathcal{A})} \nu(\mu)$.
    \end{enumerate}
\end{theorem}

Theorem \ref{thm:DSPI} (1) states that $V^{\pi_k}$ asymptotically converges to the optimal value function as long as the stepsizes are non-summable. Theorem \ref{thm:DSPI} (2) characterizes global geometric convergence, which implies the iteration complexity presented below.
\begin{corollary}
   For any $\epsilon>0$, choose $\beta_0=1$, $\beta_k=1/2$ for $k\geq1$, and choose $\tau$ such that $\tau\nu_{\max}\leq 1$ in Algorithm~\ref{alg:DSPI}. Then, we have $\|V^*-V^{\pi_k}\|_\infty\le \epsilon$ provided that $k\geq 2(1-\gamma)^{-1}\log(\epsilon^{-1}(1-\gamma)^{-1})$.
\end{corollary}

\subsection{Convergence Rates of Natural Policy Gradient and Policy Dual Averaging}\label{subsec:NPG}
As Proposition \ref{prop:equivalence} shows NPG as a special case of the DSPI algorithm, 
Theorem \ref{thm:DSPI} allows us to establish the geometric convergence of NPG below.

\begin{theorem}\label{thm:NPG}
    Consider $\{\pi_k\}$ generated by Algorithm~\ref{alg:PDA}. 
    When choosing $\alpha_0=\log(m)$ and $\alpha_{k} = \beta\alpha_0/(1-\beta)^{k}$ for all $k\geq 1$, where $\beta\in (0,1)$ can be arbitrary, we have 
        \begin{align*}
            \|V^* - V^{\pi_k}\|_\infty
            \le (1 - (1 - \gamma)\beta)^{k-1}
            \left( \gamma \|V^* - V^{\pi_0}\|_\infty + 1\right),
            \quad \forall\, k \geq 1.
        \end{align*}
\end{theorem}
\begin{remark}
The geometrically increasing stepsize sequence $\{\alpha_k\}$ has also been used in prior works to establish geometric convergence of NPG \cite{ju2026strongly,chen2025API,bhandari2021linear,lan2023policy}. Notably, Proposition~\ref{prop:equivalence} shows that the same policy sequence can be implemented via Algorithm~\ref{alg:DSPI} using the normalized recursion $\bar Q_{k+1}=(1-\beta)\bar Q_k+\beta Q^{\pi_k}$,
together with the decaying smoothing parameter $\eta_k=\tau(1-\beta)^k$. Thus, the prescribed schedule does not require storing exponentially large variables.
\end{remark}
The proof of Theorem \ref{thm:NPG} follows by combining Proposition \ref{prop:equivalence} and Theorem \ref{thm:DSPI}, and is given in Appendix \ref{pf:thm:NPG}. As a consequence, we obtain the following iteration complexity for NPG.

\begin{corollary}\label{co:NPG}
    For any $\epsilon>0$, by choosing $\alpha_0=\log(m)$ and $\alpha_k=\log(m)\cdot 2^{k-1}$ for all $k\geq 1$ in Algorithm \ref{alg:PDA}, we have $\|V^*-V^{\pi_k}\|_\infty\le \epsilon$ provided that $k\geq 2(1-\gamma)^{-1}\log(\epsilon^{-1}(1-\gamma)^{-1})$.
\end{corollary}

Theorem \ref{thm:NPG} and Corollary \ref{co:NPG} show that NPG enjoys global geometric convergence, measured by the optimality gap of the last iterate $\pi_k$. Moreover, the convergence bounds have several important features, which we discuss below.

\textit{Distribution-Free Convergence Rates.} In the existing literature on NPG, convergence rates typically depend on the initial distribution $\rho$, either explicitly or implicitly through a distribution mismatch coefficient \cite{mei2020global,shani2020adaptive,xiao2022convergence,yuan2023linear,bhandari2024global,bhandari2021linear,li2022first,cayci2024convergence}. In contrast, our convergence bound is expressed in terms of the $\ell_\infty$ norm of the value-function gap and is free of such dependencies. 

\textit{No Additional Regularization.} While our iteration complexity matches the state of the art, our result is obtained without modifying the MDP or the algorithm. Several existing analyses establishing geometric convergence for NPG either regularize the MDP \cite{geist2019theory,cen2022fast,agarwal2021theory,lan2023policy,zhan2023policy,cayci2024convergence} or introduce an additional strongly convex regularizer into the algorithm \cite{lan2023policy,li2024homotopic,lan2023block}. In contrast, we do not rely on any such modifications; the only regularization arises from the entropy term inherent in the standard NPG formulation. Our DSPI viewpoint, together with its Bellman-operator formulation (cf. Line~4 of Algorithm~\ref{alg:DSPI}), allows us to directly exploit the contraction property of the Bellman operator, thereby eliminating the need for additional curvature-inducing regularization.

\textit{A Simple Prespecified Stepsize Schedule.} Our results show that NPG converges at a geometric rate with a simple stepsize schedule $\alpha_k=\beta\alpha_0/(1-\beta)^k$, where $\beta\in (0,1)$ is arbitrary. In contrast, existing results that view NPG as an approximation of PI \cite{bhandari2021linear,chen2025API} require adaptive stepsizes that depend on the value of $\pi_k(a_k^* \mid s)$, where $a_k^*\in\argmax_{a\in\mathcal{A}} Q^{\pi_k}(s,a)$. Similarly, analyses that study NPG from an optimization viewpoint require stepsizes that depend on the Bellman error (also known as the advantage gap) of the initial policy \cite[Theorems~3.4 and~3.5]{ju2026strongly}.

\textbf{Policy Dual Averaging.} Similarly, by choosing $\nu(\mu)=C-\omega(\mu)$,
where $\omega$ is a bounded divergence-generating function and
$C:=\max_{\mu\in\Delta(\mathcal A)}\omega(\mu)$, Algorithm~\ref{alg:DSPI}
captures PDA. 
Theorem~\ref{thm:DSPI} therefore implies global geometric convergence for PDA. The results and proofs are identical to those of Theorem~\ref{thm:NPG} and are therefore omitted.

\subsection{Finite Termination of Dual-Averaged Policy Iteration}\label{subsec:DPI}
When setting $\nu(\cdot)\equiv 0$ in Algorithm \ref{alg:DSPI}, the DSPI algorithm reduces to the dual-averaged PI, which updates the policies greedily with respect to the averaged $Q$-functions. As a special case of DSPI, the dual-averaged PI also enjoys geometric convergence (cf.\ Theorem \ref{thm:DSPI} with $\nu_{\max}\equiv 0$). Moreover, we show that the algorithm terminates in a finite number of iterations with an optimal policy. This is presented in our next result.

\begin{theorem}\label{thm:strong_polynomial}
Let $\{\pi_k\}$ be generated by Algorithm~\ref{alg:DSPI} with $\nu(\cdot)\equiv 0$ and a deterministic $\pi_0$, where Line~4 selects a deterministic greedy policy according to a fixed tie-breaking rule. When $\beta_0 = 1$ and $\beta_k \equiv \beta \in (0,1)$ for all $k \geq 1$, the algorithm terminates after at most
$n(m-1)\left\lceil \beta^{-1}(1-\gamma)^{-1} \log\left( 2(1-\gamma)^{-1} \right) \right\rceil$
iterations with an optimal policy, where $\lceil x \rceil$ denotes the smallest integer greater than or equal to $x$.
\end{theorem}

\begin{remark}
When $\gamma$ is treated as fixed, this yields a strongly polynomial bound in the numbers of states and actions, in the standard sense used for discounted MDP policy-iteration analyses \cite{ye2011simplex,scherrer2014approximate}. More generally, the bound is polynomial in $n$, $m$, and the effective horizon $(1-\gamma)^{-1}$.
\end{remark}

The proof of Theorem~\ref{thm:strong_polynomial} is given in Appendix~\ref{sec:pf:strong_polynomial}. Following \cite{scherrer2014approximate}, the key idea is to show that dual-averaged PI eliminates at least one suboptimal action at some state every $k^*$ iterations, where $k^*=\lceil \beta^{-1}(1-\gamma)^{-1}\log(2(1-\gamma)^{-1}) \rceil$. Since there are at most $n(m-1)$ suboptimal actions, dual-averaged PI finds an optimal policy in at most $n(m-1)k^*$ iterations.

\subsection{Extensions}\label{subsec:extensions}
Although we have focused on discounted MDPs with a tabular representation, the generality of our framework allows both the algorithm and its analysis to extend naturally to other settings. We discuss these extensions below.

\textbf{Natural Policy Gradient with Linear Function Approximation.} To overcome the curse of dimensionality, algorithms often incorporate function approximation (e.g., approximate dynamic programming). In Appendix~\ref{sec:NPG-LFA}, we show that NPG with linear function approximation also admits a DSPI formulation, and we establish distribution-free global geometric convergence up to a function approximation error.

\textbf{Natural Policy Gradient for Stochastic Shortest Path Problems.} Beyond the discounted setting, our framework also extends to stochastic shortest path problems, a special case of undiscounted MDPs. Classical results \cite{bertsekas1996neuro} show that, under mild assumptions, the Bellman operator is monotone and contractive with respect to a weighted $\ell_\infty$ norm. Leveraging the DSPI formulation, we show that NPG for stochastic shortest path problems enjoys distribution-free global geometric convergence. See Appendix~\ref{sec:SSP} for more details.

\section{Proof of Theorem \ref{thm:DSPI}}\label{sec:proof}
Our proof proceeds in two steps: we first establish monotonic improvement of the $Q$-functions, and then derive and solve a contractive recursion for the convergence error. 

\begin{lemma}\label{le:monotone}
    Algorithm~\ref{alg:DSPI} satisfies $Q^{\pi_k} \le Q^{\pi_{k+1}}$ for all $k \ge 0$.
\end{lemma}
\begin{remark}
    By our convention, $Q^{\pi_k}\leq Q^{\pi_{k+1}}$ means that $Q^{\pi_k}(s,a)\leq Q^{\pi_{k+1}}(s,a)$ for all $(s,a)$. This vector-version monotonic improvement lemma is stronger than those in the literature on NPG, where monotonicity is established only for the scalar objective $V^{\pi_k}(\rho)$ \cite{agarwal2021theory}. This stronger form is important for leveraging the contraction property of the Bellman operator in the next step to derive a contractive recursion for the convergence error.
\end{remark}
\begin{proof}[Proof of Lemma \ref{le:monotone}]
    It suffices to show that $Q^{\pi_k} \le \mathcal{H}^{\pi_{k+1}}(Q^{\pi_k})$ for all $k \ge 0$. Once this is established, repeatedly applying $\mathcal{H}^{\pi_{k+1}}$ to both sides and using its monotonicity yields $Q^{\pi_k} \le [\mathcal{H}^{\pi_{k+1}}]^n(Q^{\pi_k})$ for all $n \ge 0$. Taking $n \to \infty$ gives $Q^{\pi_k} \le Q^{\pi_{k+1}}$.

For simplicity of notation, define $f:\mathbb{R}^{mn}\to\mathbb{R}^{mn}$ by
$[f(\pi)](s,a)=\gamma \sum_{s'\in\mathcal{S}} p(s'\mid s,a)\nu(\pi(s'))$ for all $(s,a)$. Then, for any policy $\pi$ and $Q\in\mathbb{R}^{mn}$, we have $\mathcal{H}_\eta^\pi(Q)=\mathcal{H}^\pi(Q)+\eta f(\pi)$. 
We now show that $Q^{\pi_k} \le \mathcal{H}^{\pi_{k+1}}(Q^{\pi_k})$ for the two cases $k=0$ and $k\geq 1$.

\begin{itemize}
    \item For $k=0$, since $\beta_0=1$, we have $\bar{Q}_1=Q^{\pi_0}$. It follows that $Q^{\pi_0}
    = \mathcal{H}^{\pi_0}(\bar{Q}_1)
    = \mathcal{H}_{\eta_0}^{\pi_0}(\bar{Q}_1) -\eta_0 f(\pi_0)$.
    Since $\mathcal{H}_{\eta_0}^{\pi_0}(\bar{Q}_1)\leq \mathcal{H}_{\eta_0}(\bar{Q}_1)=\mathcal{H}_{\eta_0}^{\pi_1}(\bar{Q}_1)$ and
    $f(\pi_1)\leq f(\pi_0)$, where the latter follows from initializing
    $\pi_0(s)\in\argmax_{\mu\in\Delta(\mathcal{A})}\nu(\mu)$ for all $s\in\mathcal{S}$, we further obtain
    \begin{align*}
        Q^{\pi_0}
        \le \mathcal{H}_{\eta_0}^{\pi_1}(\bar{Q}_1) - \eta_0 f(\pi_1)
        = \mathcal{H}^{\pi_1}(\bar{Q}_1)
        = \mathcal{H}^{\pi_1}(Q^{\pi_0}).
    \end{align*}
    \item 
    For $k\ge 1$, since $Q^{\pi_k}=\mathcal{H}^{\pi_k}(Q^{\pi_k})$ and $Q^{\pi_k} = (\bar{Q}_{k+1} - (1-\beta_k)\bar{Q}_k)/\beta_k$ (cf. Line~3 of Algorithm~\ref{alg:DSPI}), we have
    \begin{align*}
        Q^{\pi_k}
        =\,& \mathcal{H}_{\eta_{k-1}}^{\pi_k}\!\left(\frac{\bar{Q}_{k+1} - (1-\beta_k)\bar{Q}_k}{\beta_k}\right)
        - \eta_{k-1} f(\pi_k)\\
        =\,& \frac{1}{\beta_k}\mathcal{H}_{\eta_{k-1}}^{\pi_k}(\bar{Q}_{k+1})
        - \frac{(1-\beta_k)}{\beta_k}\mathcal{H}_{\eta_{k-1}}^{\pi_k}(\bar{Q}_k)
        - \eta_{k-1} f(\pi_k)\\
        =\,& \frac{1}{\beta_k}\mathcal{H}_{\eta_k}^{\pi_k}(\bar{Q}_{k+1})
        - \frac{(1-\beta_k)}{\beta_k}\mathcal{H}_{\eta_{k-1}}^{\pi_k}(\bar{Q}_k)
        + \left(\frac{(1-\beta_k)\eta_{k-1} - \eta_{k}}{\beta_k}\right) f(\pi_k)\\
        =\,& \frac{1}{\beta_k}\mathcal{H}_{\eta_k}^{\pi_k}(\bar{Q}_{k+1})
        - \frac{(1-\beta_k)}{\beta_k}\mathcal{H}_{\eta_{k-1}}^{\pi_k}(\bar{Q}_k),
    \end{align*}
    where the second equality follows from $\mathcal{H}_{\eta_{k-1}}^{\pi_k}$ being an affine operator, and the last equality follows from $\eta_k=(1-\beta_k)\eta_{k-1}$.

    To proceed, observe that by Line~4 of Algorithm~\ref{alg:DSPI}, we have
    $\mathcal{H}_{\eta_k}^{\pi_k}(\bar{Q}_{k+1}) \le \mathcal{H}_{\eta_k}^{\pi_{k+1}}(\bar{Q}_{k+1})$
    and
    $\mathcal{H}_{\eta_{k-1}}^{\pi_k}(\bar{Q}_k) \ge \mathcal{H}_{\eta_{k-1}}^{\pi_{k+1}}(\bar{Q}_k)$.
    Therefore,
    \begin{align*}
        Q^{\pi_k}
        \le\,& \frac{1}{\beta_k}\mathcal{H}_{\eta_k}^{\pi_{k+1}}(\bar{Q}_{k+1})
        - \frac{(1-\beta_k)}{\beta_k}\mathcal{H}_{\eta_{k-1}}^{\pi_{k+1}}(\bar{Q}_k)\\
        =\,& \mathcal{H}^{\pi_{k+1}}\!\left(\frac{\bar{Q}_{k+1} - (1-\beta_k)\bar{Q}_k}{\beta_k}\right)
        + \left(\frac{\eta_k-(1-\beta_k)\eta_{k-1}}{\beta_k}\right)f(\pi_{k+1})\\
        =\,& \mathcal{H}^{\pi_{k+1}}(Q^{\pi_k}),
    \end{align*}
    where the last equality follows from Algorithm \ref{alg:DSPI}, Line 3, and $\eta_k=(1-\beta_k)\eta_{k-1}$.
\end{itemize}
\end{proof}

With Lemma~\ref{le:monotone} in hand, we next derive a contractive recursion for $\bar{Q}_k$. 
\begin{lemma}\label{le:contraction_recursion}
    Algorithm~\ref{alg:DSPI} satisfies
    \begin{align*}
        \|Q^*-\Bar{Q}_{k+1}\|_\infty
        \leq (1-(1-\gamma)\beta_k)\|Q^*-\Bar{Q}_k\|_\infty
        +\gamma\beta_k\eta_{k-1}\nu_{\max},\quad \forall\,k\geq 1.
    \end{align*}
\end{lemma}
\begin{proof}[Proof of Lemma \ref{le:contraction_recursion}]
    An immediate implication of Lemma~\ref{le:monotone} is that $\bar{Q}_k \le Q^{\pi_k}$ for all $k \ge 0$, since $\bar{Q}_k$ is a weighted average of a monotonically increasing sequence.

    Let $\mathbf{1}$ denote the all-ones vector. 
    Using Algorithm~\ref{alg:DSPI}, Line~3, we obtain for all $k\geq 1$ that
    \begin{subequations}\label{eqeq}
\begin{align}
        Q^*-\Bar{Q}_{k+1}
        =\,&(1-\beta_k)(Q^*-\Bar{Q}_k)+\beta_k(Q^*-Q^{\pi_k})\nonumber\\
        =\,&(1-\beta_k)(Q^*-\Bar{Q}_k)
        +\beta_k(\mathcal{H}(Q^*)-\mathcal{H}^{\pi_k}(Q^{\pi_k}))\label{eqeq:a}\\
        \leq\,&(1-\beta_k)(Q^*-\Bar{Q}_k)
        +\beta_k(\mathcal{H}(Q^*)-\mathcal{H}^{\pi_k}(\Bar{Q}_k))\label{eqeq:b}\\
        =\,&(1-\beta_k)(Q^*-\Bar{Q}_k)
        +\beta_k(\mathcal{H}(Q^*)-\mathcal{H}_{\eta_{k-1}}^{\pi_k}(\Bar{Q}_k))
        +\beta_k\eta_{k-1}f(\pi_k)\nonumber\\
        \leq\,&(1-\beta_k)(Q^*-\Bar{Q}_k)
        +\beta_k(\mathcal{H}(Q^*)-\mathcal{H}(\Bar{Q}_k))
        +\beta_k\eta_{k-1}f(\pi_k)\label{eqeq:c}\\
        \leq\,&(1-\beta_k)\|Q^*-\Bar{Q}_k\|_\infty\mathbf{1}
        +\beta_k\|\mathcal{H}(Q^*)-\mathcal{H}(\Bar{Q}_k)\|_\infty\mathbf{1}
        +\gamma\beta_k\eta_{k-1}\nu_{\max}\mathbf{1}\nonumber\\
        \leq\,&\bigl(1-(1-\gamma)\beta_k\bigr)\|Q^*-\Bar{Q}_k\|_\infty\mathbf{1}
        +\gamma\beta_k\eta_{k-1}\nu_{\max}\mathbf{1},\label{eqeq:d}
    \end{align}
\end{subequations}
    where \eqref{eqeq:a} follows from the Bellman equations, \eqref{eqeq:b} follows from $\bar{Q}_k \le Q^{\pi_k}$ and $\mathcal{H}^{\pi_k}$ being a monotonic operator, \eqref{eqeq:c} follows from $\mathcal{H}(\bar{Q}_k)\leq \mathcal{H}_{\eta_{k-1}}(\bar{Q}_k)=\mathcal{H}_{\eta_{k-1}}^{\pi_{k}}(\bar{Q}_k)$, and \eqref{eqeq:d} follows from $\mathcal{H}$ being a contraction mapping with respect to $\|\cdot\|_\infty$. 
   Since the vector $Q^*-\Bar{Q}_{k+1}$ has non-negative entries, the desired inequality follows.
\end{proof}

\begin{proof}[Proof of Theorem \ref{thm:DSPI}]
    Repeatedly applying Lemma~\ref{le:contraction_recursion}, we obtain a bound on $\|Q^* - \bar{Q}_k\|_\infty$. The final step of the proof is to translate this bound into one on $\|V^* - V^{\pi_k}\|_\infty$, and then solve the resulting recursion. This step involves only algebraic manipulations, and we defer the details to Appendix~\ref{ap:translation}. The proof is complete after this step.
\end{proof}

\begin{remark}
    In view of the proof of Theorem~\ref{thm:DSPI}, although DSPI captures gradient-based policy optimization algorithms (e.g., NPG and PDA), its analysis differs fundamentally from approaches based on optimization techniques. It relies directly on Bellman’s core principles: the monotonicity of the Bellman operator yields monotonic improvement of the policies (cf. Lemma~\ref{le:monotone}), and its contraction property induces a contractive recursion for the convergence error (cf. Lemma~\ref{le:contraction_recursion}).
\end{remark}

\section{Conclusion}
We introduce DSPI, an algorithmic framework that bridges classical dynamic programming algorithms, such as PI and dual-averaged PI, with modern policy optimization methods, such as NPG and PDA. This perspective enables a unified analysis based on monotonicity and contraction properties of the Bellman operator, yielding distribution-free global geometric convergence for NPG and PDA, and establishing finite termination for dual-averaged PI. The framework also extends to discounted MDPs with linear function approximation and to stochastic shortest path problems.

Looking ahead, our current analysis assumes access to the exact $Q$-function $Q^\pi$ for each policy iterate, which isolates the policy-improvement dynamics. Extending DSPI to the model-free RL setting, where $Q^\pi$ must be estimated from sampled trajectories (e.g., via temporal-difference learning or Monte Carlo rollouts \cite{sutton2018reinforcement}), is therefore a natural next step. While finite-time guarantees for such policy evaluation methods are well understood \cite{srikant2019finite,bhandari2018finite,chen2024lyapunov}, combining them with our framework to obtain sharp sample complexity results for model-free RL remains nontrivial. In particular, determining whether these variants can achieve minimax-optimal dependence on all problem parameters, namely $\mathcal{O}(\epsilon^{-2}(1-\gamma)^{-3}mn)$ \cite{azar2013minimax}, is, to the best of our knowledge, still broadly open in the existing analysis of actor-critic methods, and is an important direction for future work.

\bibliographystyle{apalike}
\bibliography{references}

\newpage

\begin{center}
    {\LARGE\bfseries Appendices}
\end{center}

\appendix

\section{Related Work on Natural Policy Gradient}
\label{app:comparison}

\begin{table}[ht]
\caption{Summary of existing analyses of natural policy gradient.}
\label{table:1}
\centering
\small
\setlength{\tabcolsep}{4pt}
\renewcommand{\arraystretch}{1.5}
\begin{tabular}{
!{\vrule width 1.5pt}
>{\centering\arraybackslash}p{4 cm}
| >{\centering\arraybackslash}p{2.5cm}
| >{\centering\arraybackslash}p{2.3 cm}
| >{\centering\arraybackslash}p{2.5 cm}
| >{\centering\arraybackslash}p{2.4cm}
!{\vrule width 1.5pt}
}
\thickhline
&
\textbf{Global Geometric Convergence} &
\textbf{Distribution Free} &
\textbf{Free of Regularization} &
\textbf{Free of Adaptive Stepsize} \\
\thickhline
\citet{agarwal2021theory} & \xmark & \cmark & \cmark & \cmark \\
\hline
\citet{shani2020adaptive} & \xmark & \cmark & \xmark & \cmark \\
\hline
\citet{bhandari2021linear}\footnotemark[3] & \cmark & \xmark & \cmark & \xmark \\
\hline
\citet{cen2022fast} & \cmark & \cmark & \xmark & \cmark \\
\hline
\citet{lan2023policy} & \cmark & \cmark & \xmark & \cmark \\
\hline
\citet{li2024homotopic} & \cmark & \cmark & \xmark & \cmark \\
\hline
\citet{zhan2023policy} & \cmark & \cmark & \xmark & \cmark \\
\hline
\citet{xiao2022convergence} & \cmark & \xmark & \xmark & \cmark \\ 
\hline
\citet{khodadadian2021linear} & \cmark & \cmark & \cmark & \xmark \\
\hline
\citet{chen2025API} & \cmark & \cmark & \cmark & \xmark \\
\hline
\citet{ju2026strongly} & \cmark & \cmark & \cmark & \cmark\footnotemark[4] \\
\hline
\rowcolor{gray!15} 
\textbf{This Work} & \cmark & \cmark & \cmark & \cmark \\
\thickhline
\end{tabular}
\begin{flushleft}
\footnotesize
We compare our work (last row) with existing results in terms of (i) global geometric convergence of the last iterate, (ii) dependence of the guarantees on the initial distribution (explicitly or via concentrability coefficients), (iii) the use of regularization in the MDP or the algorithm, and (iv) the need for trajectory-dependent stepsizes. 
\end{flushleft}
\end{table}
\footnotetext[3]{Global geometric convergence is established under two regimes: (i) using line search for stepsize selection where the convergence bounds depend explicitly on the initial distribution, and (ii) without such dependence using adaptive stepsizes that depend on policy updates.}
\footnotetext[4]{While \cite{ju2026strongly} does not use trajectory-dependent adaptive stepsizes, its choice of stepsize depends on the Bellman error (also called the advantage gap) of the initial policy, which is not required in our result (cf. Theorem~\ref{thm:NPG}).}

\section{Proofs of All Technical Results}\label{ap:discounted}
\subsection{The Geometric Convergence of Policy Iteration}\label{ap:PI_analysis}
\begin{theorem}\label{thm:PI}
    PI \eqref{algo:PI:Bellman} satisfies $\|V^*-V^{\pi_k}\|_\infty \le \gamma^k \|V^*-V^{\pi_0}\|_\infty$ for all $k\geq 0$.
\end{theorem}
\begin{proof}[Proof of Theorem \ref{thm:PI}]
Our proof follows the two-step approach outlined in the last paragraph of Section~\ref{subsec:Bellman}.
\begin{itemize}
    \item \textbf{Step 1: Monotonic Improvement:} Using the Bellman equation, we have
$Q^{\pi_k} = \mathcal{H}^{\pi_k}(Q^{\pi_k}) \le \mathcal{H}(Q^{\pi_k}) = \mathcal{H}^{\pi_{k+1}}(Q^{\pi_k})$,
where the last equality follows from the update rule~\eqref{algo:PI:Bellman}. Repeatedly applying $\mathcal{H}^{\pi_{k+1}}$ to both sides of
$Q^{\pi_k} \le \mathcal{H}^{\pi_{k+1}}(Q^{\pi_k})$ and using its monotonicity yields
$Q^{\pi_k} \le [\mathcal{H}^{\pi_{k+1}}]^n(Q^{\pi_k})$ for all $n \ge 0$.
Letting $n \to \infty$ gives
$Q^{\pi_k} \le Q^{\pi_{k+1}}$.

\item \textbf{Step 2: A Contractive Recursion:} Since $Q^{\pi_k} \le Q^{\pi_{k+1}}$ and $\mathcal{H}^{\pi_{k+1}}$ is monotonic, we have
\begin{align*}
    Q^* - Q^{\pi_{k+1}}
    = \mathcal{H}(Q^*) - \mathcal{H}^{\pi_{k+1}}(Q^{\pi_{k+1}})
    \le \mathcal{H}(Q^*) - \mathcal{H}^{\pi_{k+1}}(Q^{\pi_k})
    = \mathcal{H}(Q^*) - \mathcal{H}(Q^{\pi_k}).
\end{align*}
Since both sides are non-negative vectors, we obtain
\begin{align*}
    \|Q^* - Q^{\pi_{k+1}}\|_\infty
    \le \|\mathcal{H}(Q^*) - \mathcal{H}(Q^{\pi_k})\|_\infty
    \le \gamma \|Q^* - Q^{\pi_k}\|_\infty,\quad \forall\,k\geq 0.
\end{align*}
Solving this recursion yields $\|Q^* - Q^{\pi_k}\|_\infty \le \gamma^k\|Q^* - Q^{\pi_0}\|_\infty$ for all $k \ge 0$.
\end{itemize}
To translate the $Q$-function gap $\|Q^{\pi_k} - Q^*\|_\infty$to the value function gap $\|V^{\pi_k} - V^*\|_\infty$, observe that on the one hand, we have
\begin{align*}
    V^*(s) - V^{\pi_k}(s)
    =\,& \max_{a \in \mathcal{A}} Q^*(s,a) - \pi_k(s)^\top Q^{\pi_k}(s) \\
    \le\,& \max_{a \in \mathcal{A}} Q^*(s,a) - \pi_k(s)^\top Q^{\pi_{k-1}}(s)
    \tag{$Q^{\pi_{k-1}} \le Q^{\pi_k}$ for all $k \ge 1$} \\
    =\,& \max_{a \in \mathcal{A}} Q^*(s,a) - \max_{a \in \mathcal{A}} Q^{\pi_{k-1}}(s,a)
    \tag{update rule \eqref{algo:PI:Bellman}} \\
    \le\,& \|Q^* - Q^{\pi_{k-1}}\|_\infty .
\end{align*}
On the other hand, we have
\begin{align*}
    Q^*(s,a) - Q^{\pi_0}(s,a)
    = \gamma \sum_{s' \in \mathcal{S}} p(s' \mid s,a)
    \bigl(V^*(s') - V^{\pi_0}(s')\bigr)
    \le \gamma \|V^* - V^{\pi_0}\|_\infty .
\end{align*}
Combining the above two inequalities with the convergence rate of $\|Q^* - Q^{\pi_k}\|_\infty$ established in Step~2 yields $\|V^* - V^{\pi_k}\|_\infty
    \le \gamma^k \|V^* - V^{\pi_0}\|_\infty$ for all $k\geq 0$.
\end{proof}

\subsection{Properties of the Smoothed Bellman Operators}\label{ap:smooth_Bellman_properties}
\begin{lemma}\label{le:smooth_Bellman_properties}
    For any policy $\pi$ and $\eta \geq 0$, let $F$ denote either $\mathcal{H}_\eta$ or $\mathcal{H}_\eta^\pi$. Then, for any $Q_1, Q_2 \in \mathbb{R}^{mn}$, we have
\begin{enumerate}[(1)]
    \item $\|F(Q_1) - F(Q_2)\|_\infty \leq \gamma \|Q_1 - Q_2\|_\infty$;
    \item $F(Q_1) \leq F(Q_2)$ whenever $Q_1 \leq Q_2$.
\end{enumerate}
\end{lemma}
\begin{proof}[Proof of Lemma \ref{le:smooth_Bellman_properties}]
    We prove the result only for the smoothed Bellman optimality operator $\mathcal{H}_\eta$. The proof for $\mathcal{H}_\eta^\pi$ follows by a similar argument.
    \begin{enumerate}[(1)]
        \item For any $(s,a)$, we have
        \begin{align*}
            &\left|[\mathcal{H}_\eta(Q_1)](s,a) - [\mathcal{H}_\eta(Q_2)](s,a)\right| \\
            =\,& \gamma \left|\sum_{s' \in \mathcal{S}} p(s' \mid s,a)
            \left(
            \max_{\mu \in \Delta(\mathcal{A})} \{\mu^\top Q_1(s') + \eta \nu(\mu)\}
            - \max_{\mu \in \Delta(\mathcal{A})} \{\mu^\top Q_2(s') + \eta \nu(\mu)\}
            \right)\right| \\
            \leq\,& \gamma \sum_{s' \in \mathcal{S}} p(s' \mid s,a)
            \max_{\mu \in \Delta(\mathcal{A})}
            \left|\mu^\top (Q_1(s') - Q_2(s'))\right| \\
            \leq\,& \gamma \|Q_1 - Q_2\|_\infty.
        \end{align*}
        Since the above bound holds for all $(s,a)$, we obtain
        \begin{align*}
            \|\mathcal{H}_\eta(Q_1) - \mathcal{H}_\eta(Q_2)\|_\infty
            \leq \gamma \|Q_1 - Q_2\|_\infty.
        \end{align*}

        \item For any $(s,a)$, we have
        \begin{align*}
            [\mathcal{H}_\eta(Q_1)](s,a)
            =\,& \mathcal{R}(s,a) + \gamma \sum_{s' \in \mathcal{S}} p(s' \mid s,a)
            \max_{\mu \in \Delta(\mathcal{A})} \{\mu^\top Q_1(s') + \eta \nu(\mu)\} \\
            \leq\,& \mathcal{R}(s,a) + \gamma \sum_{s' \in \mathcal{S}} p(s' \mid s,a)
            \max_{\mu \in \Delta(\mathcal{A})} \{\mu^\top Q_2(s') + \eta \nu(\mu)\} \\
            =\,& [\mathcal{H}_\eta(Q_2)](s,a).
        \end{align*}
    \end{enumerate}
\end{proof}

\subsection{Proof of Theorem \ref{thm:NPG}}\label{pf:thm:NPG}
The proof follows by combining Proposition~\ref{prop:equivalence} with Theorem~\ref{thm:DSPI}. By the definition of $\alpha_k$, we have $\beta_0 = \alpha_0 / \alpha_0 = 1$, and for every $k \geq 1$,
\begin{align*}
    \beta_{k} 
    = \frac{\alpha_k}{\sum_{i=0}^k \alpha_i} 
    = \frac{\beta\alpha_0/(1-\beta)^{k}}{\alpha_0 + \sum_{i=1}^{k}\beta\alpha_0/(1-\beta)^{i}} 
    = \frac{\beta\alpha_0/(1-\beta)^{k}}{\alpha_0 + \alpha_0 \left((1-\beta)^{-k}-1\right)} 
    = \beta.
\end{align*}
Further, we choose $\alpha_0 = \log(m)$, $\tau = 1/\alpha_0$, and $\nu(\cdot) = h(\cdot)$. Thus, by Theorem~\ref{thm:DSPI}, for every $k \geq 1$,
\begin{align*}
    \|V^* - V^{\pi_k}\|_\infty 
    \leq\;& (1 - (1 - \gamma)\beta)^{k-1} \left( \gamma \|V^* - V^{\pi_0}\|_\infty + \frac{\max_{\mu\in\Delta(\mathcal{A})}h(\mu)}{\log(m)} \right) \\
    =\;& (1 - (1 - \gamma)\beta)^{k-1} \left( \gamma \|V^* - V^{\pi_0}\|_\infty + 1 \right).
\end{align*}

\subsection{Details in the Proof of Theorem~\ref{thm:DSPI}}\label{ap:translation}
Repeatedly apply Lemma \ref{le:contraction_recursion} yields
\begin{align}\label{eq:barQ_gap}
    \|Q^*-\Bar{Q}_{k+1}\|_\infty \leq\prod_{j=1}^{k}\bigl(1-(1-\gamma)\beta_j\bigr) \|Q^*-\bar{Q}_{1}\|_\infty + \gamma\nu_{\max}\sum_{i=1}^{k}\eta_{i-1}\beta_{i} \prod_{j=i+1}^{k}\bigl(1-(1-\gamma)\beta_{j}\bigr)
\end{align}
To translate this bound into a bound on the value function gap, we use the following lemma.
\begin{lemma}\label{le:translation}
    Algorithm~\ref{alg:DSPI} satisfies:
    \begin{enumerate}[(1)]
        \item $\|V^*-V^{\pi_{k+1}}\|_\infty\leq \|Q^*-\Bar{Q}_{k+1}\|_\infty +\eta_k\nu_{\max}$ for all $k\geq 0$;
        \item $\|Q^*-\bar{Q}_1\|_\infty\leq \gamma\|V^*-V^{\pi_0}\|_\infty$.
    \end{enumerate}
\end{lemma}
\begin{proof}[Proof of Lemma \ref{le:translation}]
    The proof uses the relation between the value function and corresponding $Q$-function, and the update rule of Algorithm \ref{alg:DSPI}. 
    \begin{enumerate}[(1)]
    \item For any $k\geq 0$ and $s\in\mathcal{S}$, we have
    \begin{align*}
        V^*(s)-V^{\pi_{k+1}}(s)
        =\,&\max_{\mu\in \Delta(\mathcal{A})}\mu^\top Q^*(s)-\pi_{k+1}(s)^\top Q^{\pi_{k+1}}(s)\\
        \leq\,&\max_{\mu\in \Delta(\mathcal{A})}\mu^\top Q^*(s)-\pi_{k+1}(s)^\top \Bar{Q}_{k+1}(s)\tag{Lemma \ref{le:monotone}}\\
        =\,&\max_{\mu\in \Delta(\mathcal{A})}\mu^\top Q^*(s)
        -\max_{\mu\in\Delta(\mathcal{A})}\left\{\mu^\top \Bar{Q}_{k+1}(s)+\eta_k\nu(\mu)\right\}
        +\eta_k\nu(\pi_{k+1}(s))\tag{Eq. \eqref{eq:DSPI_explicit}}\\
        \leq\,&\max_{\mu\in \Delta(\mathcal{A})}\mu^\top (Q^*(s)-\Bar{Q}_{k+1}(s))
        +\eta_k\nu(\pi_{k+1}(s))\\
        \leq\,&\|Q^*-\Bar{Q}_{k+1}\|_\infty 
        +\eta_k\nu_{\max}.
    \end{align*}
    Since the previous inequality holds for all $s\in\mathcal{S}$ and $V^*-V^{\pi_{k+1}}\geq 0$, we have
    \begin{align*}
        \|V^*-V^{\pi_{k+1}}\|_\infty
        \leq \|Q^*-\Bar{Q}_{k+1}\|_\infty +\eta_k\nu_{\max},\quad \forall\,k\geq 0.
    \end{align*}
    \item Since $\beta_0=1$, we have $\bar{Q}_1=Q^{\pi_0}$. Now, for any $(s,a)$, observe that
    \begin{align*}
        Q^*(s,a)-Q^{\pi_0}(s,a)=\gamma \sum_{s'\in\mathcal{S}}p(s'\mid s,a) (V^*(s')- V^{\pi_0}(s'))
        \leq \gamma \|V^*-V^{\pi_0}\|_\infty,
    \end{align*}
    which implies $\|Q^*-Q^{\pi_0}\|_\infty\leq \gamma \|V^*-V^{\pi_0}\|_\infty$. 
    \end{enumerate}
\end{proof}
Combining Lemma \ref{le:translation} with \eqref{eq:barQ_gap}, we have for every $k \geq 0$ that
\begin{align}\label{eq:final_recursion1}
    \|V^*-V^{\pi_{k+1}}\|_\infty
    \leq\,& \prod_{j=1}^{k}(1-(1-\gamma)\beta_j)\gamma \|V^*-V^{\pi_0}\|_\infty \nonumber\\
    & + \gamma\nu_{\max}\sum_{i=1}^{k}\beta_{i}\eta_{i-1}\prod_{j=i+1}^{k}(1-(1-\gamma)\beta_{j})+\eta_{k}\nu_{\max}.
\end{align}
We now use the preceding inequality to establish both asymptotic convergence and convergence rates for DSPI.

\subsubsection{The Asymptotic Convergence}\label{ap:convergence}
Since $\eta_k = \tau\prod_{j=1}^{k}(1-\beta_j)$ for all $k \geq 1$, we can write (\ref{eq:final_recursion1}) as 
\begin{align*}
    \|V^*-V^{\pi_{k+1}}\|_\infty\leq\,& \prod_{j=1}^{k}(1-(1-\gamma)\beta_j)\gamma \|V^*-V^{\pi_0}\|_\infty\\
    &+\gamma\nu_{\max}\tau\sum_{i=1}^{k}\beta_i\prod_{j=1}^{i-1}(1-\beta_j)\prod_{j=i+1}^{k}(1-(1-\gamma)\beta_{j})+\nu_{\max}\tau\prod_{j=1}^{k}(1-\beta_j)\\
    \leq \,&\prod_{j=1}^{k}(1-(1-\gamma)\beta_j)\gamma \|V^*-V^{\pi_0}\|_\infty\\
    &+\gamma\tau\nu_{\max}\prod_{j=1}^{k}(1-(1-\gamma)\beta_{j})\sum_{i=1}^{k}\frac{\beta_i}{1-(1-\gamma)\beta_i}+\nu_{\max}\tau\prod_{j=1}^{k}(1-\beta_j)\\
    \leq \,&\exp\left(-(1-\gamma)\sum_{i=1}^{k}\beta_i\right)\gamma \|V^*-V^{\pi_0}\|_\infty\\
    &+\gamma\tau\nu_{\max}\exp\left(-(1-\gamma)\sum_{i=1}^{k}\beta_{i}\right)\left(\sum_{i=1}^{k}\beta_i\right)+\nu_{\max}\tau\exp\left(-\sum_{i=1}^{k}\beta_{i}\right),
\end{align*}
where the last inequality follows from $\beta_k\in (0,1]$ and $1+x\leq e^x$ for any $x\in\mathbb{R}$. Since $\lim_{x\rightarrow \infty}e^{-x}=0$ and $\lim_{x\rightarrow \infty}xe^{-x}=0$, as long as $\sum_{k=0}^{\infty}\beta_k=\infty$, we have $\lim_{k\rightarrow\infty}\|V^*-V^{\pi_k}\|_\infty=0$.

\subsubsection{Convergence Rates}\label{ap:algebra}
We begin with Eq. \eqref{eq:final_recursion1} restated below:
\begin{align*}
    \|V^*-V^{\pi_{k+1}}\|_\infty \leq\,& \underbrace{\prod_{j=1}^{k}(1-(1-\gamma)\beta_j)\gamma \|V^*-V^{\pi_0}\|_\infty}_{:=E_1} \nonumber\\
    & +\underbrace{\gamma\nu_{\max}\sum_{i=1}^{k}\beta_{i}\eta_{i-1}\prod_{j=i+1}^{k}(1-(1-\gamma)\beta_{j})}_{:=E_2}+\underbrace{\eta_{k}\nu_{\max}}_{:=E_3}.
\end{align*}
When using $\beta_0=1$ and $\beta_k\equiv\beta\in (0,1]$ for all $k\geq 1$, we have
\begin{align*}
    E_1=(1-(1-\gamma)\beta)^k\gamma \|V^*-V^{\pi_0}\|_\infty,\quad \text{and}\quad E_3=\tau(1-\beta)^k\nu_{\max}.
\end{align*}
For the term $E_2$, we have
\begin{align*}
    E_2=\,&\gamma\nu_{\max}\sum_{i=1}^{k}\beta_{i}\eta_{i-1}\prod_{j=i+1}^{k}(1-(1-\gamma)\beta_{j})\\
    =\,&\gamma\beta\nu_{\max}\tau\sum_{i=1}^{k}(1-\beta)^{i-1}(1-(1-\gamma)\beta)^{k-i}\\
    =\,&\gamma\beta\nu_{\max}\tau (1-(1-\gamma)\beta)^{k-1}\sum_{i=1}^{k}\left(\frac{1-\beta}{1-(1-\gamma)\beta}\right)^{i-1}\\
    =\,&\gamma\beta\nu_{\max}\tau(1-(1-\gamma)\beta)^{k-1}\frac{1-\left(\frac{1-\beta}{1-(1-\gamma)\beta}\right)^k}{1-\left(\frac{1-\beta}{1-(1-\gamma)\beta}\right)}\\
    =\,&\nu_{\max}\tau\left[(1-(1-\gamma)\beta)^{k}-(1-\beta)^{k}\right].
\end{align*}
Altogether, we have
\begin{align*}
    \|V^*-V^{\pi_{k+1}}\|_\infty\leq E_1+E_2+E_3\leq (1-(1-\gamma)\beta)^k\left(\gamma\|V^*-V^{\pi_0}\|_\infty+\nu_{\max}\tau\right),\quad \forall\,k\geq 0.
\end{align*}

\subsection{Proof of Theorem \ref{thm:strong_polynomial}}\label{sec:pf:strong_polynomial}
As discussed, dual-averaged PI (DPI) is a particular case of DSPI corresponding to $\nu(\cdot)\equiv 0$; we present it below for clarity. Moreover, we restrict the iteration to deterministic policies in order to establish finite termination.
\begin{algorithm}[ht]\caption{Dual-Averaged Policy Iteration}\label{alg:DPI}
	\begin{algorithmic}[1]
		\STATE \textbf{Input:} Initialize $\pi_0$ at a deterministic policy and $\Bar{Q}^0=\mathbf{0}$
		\FOR{$k=0,1,2,\cdots$}
        \STATE Maintain a running average of past $Q$-functions: $\bar{Q}_{k+1}=(1-\beta_k)\bar{Q}_k+\beta_k Q^{\pi_k}$,
        where $\beta_k\in (0,1]$ is the stepsize.
        \STATE Choose a deterministic $\pi_{k+1}$ such that  $\mathcal{H}^{\pi_{k+1}}(\bar{Q}_{k+1})=\mathcal{H}(\bar{Q}_{k+1})$.
		\ENDFOR
	\end{algorithmic}
\end{algorithm}

The key to showing that DPI terminates in a finite number of iterations is to identify a $k^* \in \mathbb{N}$ such that, every $k^*$ iterations, at least one suboptimal action is eliminated in some state. To this end, we first construct a certificate that detects when two deterministic policies differ; the following lemma provides such a certificate.

\begin{lemma}\label{le:DPI_certificate}
    Let $\pi_0$ and $\pi_1$ be two deterministic policies. 
    Let $(s,a)\in\mathcal{S}\times\mathcal{A}$ be such that $Q^*(s,a)-[\mathcal{H}^{\pi_0}(Q^*)](s,a) \neq Q^*(s,a)-[\mathcal{H}^{\pi_1}(Q^*)](s,a)$. 
    Then there exists $s'\in\mathcal{S}$ for which $\pi_1(a_0'|s') = 0$, where $a_0'\in\mathcal{A}$ is such that $\pi_0(a_0'|s')=1$. 
\end{lemma}
\begin{proof}[Proof of Lemma \ref{le:DPI_certificate}]
    By the definition of the Bellman operators $\mathcal{H}^{\pi_0}$ and $\mathcal{H}^{\pi_1}$, the condition 
    $Q^*(s,a)-[\mathcal{H}^{\pi_0}(Q^*)](s,a) \neq Q^*(s,a)-[\mathcal{H}^{\pi_1}(Q^*)](s,a)$ 
    is equivalent to 
    \begin{align*}
        \sum_{s'\in\mathcal{S}}p(s'|s,a)\sum_{a'\in\mathcal{A}}\pi_0(a'|s') Q^*(s',a') \neq 
        \sum_{s'\in\mathcal{S}}p(s'|s,a)\sum_{a'\in\mathcal{A}}\pi_1(a'|s') Q^*(s',a').
    \end{align*}
    This implies that there exists $s'\in\mathcal{S}$ for which 
    \begin{align*}
        \sum_{a'\in\mathcal{A}}\pi_0(a'|s') Q^*(s',a') \neq \sum_{a'\in\mathcal{A}}\pi_1(a'|s') Q^*(s',a').
    \end{align*}
    Let $a_1'\in\mathcal{A}$ be such that $\pi_1(a_1'|s') = 1$. Then the above inequality reduces to 
    $Q^*(s',a_0') \neq Q^*(s',a_1')$, which implies $a_0' \neq a_1'$. Hence, $\pi_1(a_0'|s') = 0$. 
\end{proof}
Our next result gives a bound on the certificate corresponding to the policies generated by the DPI algorithm.
\begin{lemma}\label{le:DPI_certificate_decrease}
    Let $\{\pi_k\}$ be the sequence of policies generated by Algorithm \ref{alg:DPI}. 
    Then, for every $k \geq 1$,
    \begin{align*}
        \|Q^* - \mathcal{H}^{\pi_k}(Q^*)\|_{\infty} \leq \frac{2\gamma}{1-\gamma} \prod_{i=1}^{k-1}\bigl(1-(1-\gamma)\beta_i\bigr)\|Q^{*} - \mathcal{H}^{\pi_0}(Q^*)\|_{\infty}.
    \end{align*}
    As a result, when choosing $\beta_0 = 1$ and $\beta_k = \beta \in (0,1]$ for every $k \geq 1$, we have
    \begin{align*}
        \|Q^* - \mathcal{H}^{\pi_k}(Q^*)\|_{\infty} < \|Q^{*} - \mathcal{H}^{\pi_0}(Q^*)\|_{\infty},\;\quad \forall\, k \geq k^*:=\left\lceil \frac{1}{\beta(1-\gamma)} \log\left( \frac{2}{1-\gamma} \right) \right\rceil.
    \end{align*}
\end{lemma}
\begin{proof}[Proof of Lemma \ref{le:DPI_certificate_decrease}]
    For any $k \geq 1$, we have 
    \begin{align*}
        Q^* - \mathcal{H}^{\pi_k}(Q^*) 
        = \;& \mathcal{H}(Q^*) - \mathcal{H}^{\pi_k}(\bar{Q}_{k}) + \mathcal{H}^{\pi_k}(\bar{Q}_{k}) - \mathcal{H}^{\pi_k}(Q^*) \\
        = \;& \mathcal{H}(Q^*) - \mathcal{H}(\bar{Q}_{k}) + \mathcal{H}^{\pi_k}(\bar{Q}_{k}) - \mathcal{H}^{\pi_k}(Q^*) \tag{Algorithm \ref{alg:DPI}, Line 4}\\
        \leq \;& \|\mathcal{H}(Q^*) - \mathcal{H}(\bar{Q}_{k})\|_{\infty}\mathbf{1} 
        + \|\mathcal{H}^{\pi_k}(\bar{Q}_{k}) - \mathcal{H}^{\pi_k}(Q^*)\|_{\infty} \mathbf{1} \\
        \leq \;& \gamma \|Q^* - \bar{Q}_{k}\|_{\infty}\mathbf{1} + \gamma \|\bar{Q}_{k} - Q^*\|_{\infty}\mathbf{1} \\
        = \;& 2\gamma \|Q^* - \bar{Q}_{k}\|_{\infty}\mathbf{1}.
    \end{align*}
    Since $\mathcal{H}(Q) \geq \mathcal{H}^{\pi}(Q)$ for any policy $\pi$ and any $Q\in\mathbb{R}^{mn}$, we have $Q^* - \mathcal{H}^{\pi_k}(Q^*) \geq Q^* - \mathcal{H}(Q^*) = 0$. Hence,
    \begin{align}
        \|Q^* - \mathcal{H}^{\pi_k}(Q^*)\|_{\infty} \leq 2\gamma \|Q^* - \bar{Q}_{k}\|_{\infty},\quad \forall\, k \geq 1. \label{eq:API-bound}
    \end{align}
    As the DPI algorithm corresponds to $\nu(\cdot) \equiv 0$, the bound \eqref{eq:barQ_gap} reduces to
    \begin{align*}
        \|Q^*-\bar{Q}_{k+1}\|_\infty \leq \prod_{i=1}^{k}\bigl(1-(1-\gamma)\beta_i\bigr)\|Q^*-\bar{Q}_1\|_\infty,\quad \forall\, k\geq 1.
    \end{align*}
    Combining this with \eqref{eq:API-bound} gives
    \begin{align}\label{eq:DPI-certificate-bound}
        \|Q^* - \mathcal{H}^{\pi_{k}}(Q^*)\|_{\infty} 
        \leq 2\gamma \prod_{i=1}^{k-1}\bigl(1-(1-\gamma)\beta_i\bigr)\|Q^*-Q^{\pi_0}\|_\infty,
    \end{align}
    where we used $\bar{Q}_{1} = Q^{\pi_0}$.
    To relate the right-hand side to the Bellman residual $\|Q^* - \mathcal{H}^{\pi_0}(Q^*)\|_{\infty}$, we use
    \begin{align*}
        \|Q^{*} - Q^{\pi_0}\|_{\infty} 
        = &\; \|Q^{*} - \mathcal{H}^{\pi_0}(Q^*) + \mathcal{H}^{\pi_0}(Q^*) - Q^{\pi_0}\|_{\infty} \\
        = &\; \|Q^{*} - \mathcal{H}^{\pi_0}(Q^*) + \mathcal{H}^{\pi_0}(Q^*) - \mathcal{H}^{\pi_0}(Q^{\pi_0})\|_{\infty} \\
        \leq &\; \|Q^{*} - \mathcal{H}^{\pi_0}(Q^*)\|_{\infty} 
        + \|\mathcal{H}^{\pi_0}(Q^*) - \mathcal{H}^{\pi_0}(Q^{\pi_0})\|_{\infty} \\
        \leq &\; \|Q^{*} - \mathcal{H}^{\pi_0}(Q^*)\|_{\infty} + \gamma \|Q^* - Q^{\pi_0}\|_{\infty},
    \end{align*}
    which implies
    \begin{align*}
        \|Q^{*} - Q^{\pi_0}\|_{\infty} \leq \frac{1}{1-\gamma}\|Q^{*} - \mathcal{H}^{\pi_0}(Q^*)\|_{\infty}.
    \end{align*}
    Applying this bound in \eqref{eq:DPI-certificate-bound} yields
    \begin{align*}
        \|Q^* - \mathcal{H}^{\pi_k}(Q^*)\|_{\infty} 
        \leq \frac{2\gamma}{1-\gamma} \prod_{i=1}^{k-1}\bigl(1-(1-\gamma)\beta_i\bigr)\|Q^{*} - \mathcal{H}^{\pi_0}(Q^*)\|_{\infty},\quad \forall\,k\geq 0.
    \end{align*}
    For $\beta_ k\equiv \beta$, the previous bound implies
    \begin{align*}
        \|Q^* - \mathcal{H}^{\pi_k}(Q^*)\|_{\infty} \leq  \frac{2}{1-\gamma} e^{-(1-\gamma)\beta(k-1)} \|Q^{*} - \mathcal{H}^{\pi_0}(Q^*)\|_{\infty}, 
    \end{align*}
    where the last inequality follows from the fact that $1 + x \leq e^{x}$ for all $x \in \mathbb{R}$. 
   The proof is completed by observing that our choice of $k^* \in \mathbb{N}$ satisfies $\frac{2}{1-\gamma} \exp(-(1-\gamma)\beta k^*) < 1$.
\end{proof}

Let $(s,a)\in\mathcal{S}\times\mathcal{A}$ be such that $\|Q^{*} - \mathcal{H}^{\pi_0}(Q^*)\|_{\infty} = Q^{*}(s,a) - [\mathcal{H}^{\pi_0}(Q^*)](s,a)$.
Since 
\begin{align*}
    Q^*(s,a) - [\mathcal{H}^{\pi_k}(Q^*)](s,a) < Q^{*}(s,a) - [\mathcal{H}^{\pi_0}(Q^*)](s,a)
\end{align*}
for any $k\geq k^*$, Lemma \ref{le:DPI_certificate} implies that there exists a state $s'\in\mathcal{S}$ such that if $\pi_{0}(a'|s') = 1$ then $\pi_{k}(a'|s') = 0$ for every $k \geq k^*$. Thus the action $a'$ in state $s'$ will not be selected after $k^*$ iterations. Since DPI eliminates at least one suboptimal action at some state every $k^*$ iterations, and there are at most $n(m-1)$ suboptimal actions in total, it follows that DPI finds an optimal policy after at most $n(m-1)k^*$ iterations.

\section{Natural Policy Gradient with Linear Function Approximation}\label{sec:NPG-LFA}
In this section, we analyze the NPG algorithm with linear function approximation. In this case, the policy is parameterized as
\begin{align*}
    \pi_{\theta}(a|s) = \frac{\exp(\phi_{s,a}^{\top} \theta)}{\sum_{a'\in\mathcal{A}}\exp(\phi_{s,a'}^{\top} \theta)}, \quad \forall\,(s,a)\in\mathcal{S}\times \mathcal{A},
\end{align*}
where $\theta\in\mathbb{R}^d$ is the parameter vector ($d\leq  mn$), and $\phi_{s,a}\in\mathbb{R}^{d}$ denotes the feature vector associated with the state-action pair $(s,a)$. Let $\Phi \in \mathbb{R}^{mn\times d}$ be the matrix whose rows are the feature vectors. We assume without loss of generality that the columns of $\Phi$ are linearly independent.

We next present NPG with linear function approximation \cite{agarwal2021theory}, where $d_\rho^\pi \in \Delta(\mathcal{S})$ denotes the discounted state visitation distribution with initial state distribution $\rho$.

\begin{algorithm}[ht]\caption{Natural Policy Gradient with Linear Function Approximation}\label{alg:Q-NPG}
	\begin{algorithmic}[1]
		\STATE \textbf{Input:} Initialize $\theta_{0}=\mathbf{0}$ and $\pi_0$ as the uniform policy.
		\FOR{$k=0,1,2,3,\cdots$}
        \STATE $w_{k}^* = \argmin_{w\in\mathbb{R}^d}\mathbb{E}_{s\sim d_{\rho}^{\pi_{k}}, a\sim \pi_{k}(\cdot\mid s)} \left[ Q^{\pi_k}(s,a) - \phi_{s,a}^{\top} w \right]^{2}$
        \STATE $\theta_{k+1} = \theta_{k} + \alpha_{k} w_{k}^{*}$
        \STATE $\pi_{k+1}(s) = \argmax_{\mu\in\Delta(\mathcal{A})} \left\{ \mu^{\top} [\Phi\theta_{k+1}](s) + h(\mu) \right\}$ for any $s \in \mathcal{S}$
		\ENDFOR
	\end{algorithmic}
\end{algorithm} 

\subsection{Algorithm Reformulation}

Similar to the tabular setting, we reformulate the algorithm as DSPI with linear function approximation, as given below.

\begin{algorithm}[ht]\caption{Doubly Smoothed Policy Iteration with Linear Function Approximation}\label{alg:DSPI-LFA}
	\begin{algorithmic}[1]
		\STATE \textbf{Input:} Initialize $\theta_{0}=\mathbf{0}$, $\bar{W}_0=\mathbf{0}$, and $\pi_0(s)\in\argmax_{\mu\in\Delta(\mathcal{A})}\nu(\mu)$ for any $s\in\mathcal{S}$.
		\FOR{$k=0,1,2,3,\cdots$}
        \STATE $w_{k}^* = \argmin_{w}\mathbb{E}_{s\sim d_{\rho}^{\pi_{k}}, a\sim \pi_{k}(\cdot\mid s)} \left[ Q^{\pi_k}(s,a) - \phi_{s,a}^{\top} w \right]^{2}$ and $W_{k}^* = \Phi w_{k}^*$
        \STATE $\bar{W}_{k+1} = (1-\beta_{k}) \bar{W}_{k} + \beta_{k} W_{k}^*$
        \STATE $\pi_{k+1}(s) \in \argmax_{\mu\in\Delta} \left\{ \mu^{\top} \bar{W}_{k+1}(s) + \tau\prod_{j=1}^k(1-\beta_j)\nu(\mu) \right\}$ for any $s \in \mathcal{S}$
		\ENDFOR
	\end{algorithmic}
\end{algorithm} 

The following result is analogous to Proposition~\ref{prop:equivalence} 
\begin{proposition}\label{prop:equivalence-LFA}
    When $\beta_k = \alpha_k / \sum_{i=0}^k \alpha_i$, $\tau = 1/\alpha_0$, and $\nu(\cdot)=h(\cdot)$, 
    the sequence of policies generated by Algorithm~\ref{alg:DSPI-LFA} coincides with that generated by Algorithm~\ref{alg:Q-NPG}.
\end{proposition}

\begin{proof}[Proof of Proposition \ref{prop:equivalence-LFA}]
    Denote $\Theta_{k} = \Phi \theta_{k}$. 
    As stated in Line 3 of Algorithm \ref{alg:DSPI-LFA}, $W_{k}^* = \Phi w_{k}^*$. 
    Then Line 4 of Algorithm \ref{alg:Q-NPG} is equivalent to $\Theta_{k+1} = \Theta_{k} + \alpha_{k} W_{k}^{*}$, 
    since $\Phi$ has linearly independent columns. 
    Recursing yields $\Theta_{k+1} = \sum_{i=0}^{k} \alpha_{i} W_{i}^{*}$. 

    Following arguments similar to those in Proposition \ref{prop:equivalence}, we obtain 
    \begin{align*}
        \pi_{k+1}(s) 
        = &\; \argmax_{\mu \in \Delta(\mathcal{A})} \left\{ \mu^{\top} \left(\sum_{i=0}^{k} \alpha_{i} W_{i}^{*}(s) \right) + h(\mu) \right\} \\
        = &\; \argmax_{\mu \in \Delta(\mathcal{A})} \left\{ \mu^{\top} \left(\frac{\sum_{i=0}^{k} \alpha_{i} W_{i}^{*}(s)}{\sum_{i=0}^{k} \alpha_{i}}\right) 
        + \frac{\tau \alpha_{0}}{\sum_{i=0}^{k} \alpha_{i}} h(\mu) \right\},\quad \forall\, s \in \mathcal{S}.
    \end{align*}
    For the above to be equivalent to Line 5 of Algorithm \ref{alg:DSPI-LFA}, it remains to show that 
    $\Bar{W}_{k+1} = \sum_{i=0}^{k} \alpha_{i} W_{i}^{*} / \sum_{i=0}^{k} \alpha_{i}$ and 
    $\prod_{j=1}^k(1-\beta_j) = \alpha_0 / \sum_{i=0}^{k} \alpha_{i}$. 
    The latter identity has already been proved in Proposition \ref{prop:equivalence} and is omitted here.

    We now prove the former by induction, following the same argument as in the proof of Proposition \ref{prop:equivalence}. 
    The equality holds for $k = 0$, since $\beta_0 = 1$, which gives $\Bar{W}_{1} = W_{0}^{*}$. 
    Assume that $\Bar{W}_{n+1} = \sum_{i=0}^{n} \alpha_{i} W_{i}^{*} / \sum_{i=0}^{n} \alpha_{i}$ for some $n \in \mathbb{N}$. Then,
    \begin{align*}
        \Bar{W}_{n+2} 
        = &\; (1-\beta_{n+1}) \Bar{W}_{n+1} + \beta_{n+1} W_{n+1}^* \\
        = &\; \frac{\sum_{i=0}^{n}\alpha_i}{\sum_{i=0}^{n+1}\alpha_{i}} \cdot \frac{\sum_{i=0}^{n} \alpha_{i} W_{i}^{*}}{\sum_{i=0}^{n} \alpha_{i}}  
        + \frac{\alpha_{n+1}}{\sum_{i=0}^{n+1}\alpha_{i}} W_{n+1}^* \\
        = &\; \frac{\sum_{i=0}^{n+1} \alpha_{i} W_{i}^{*}}{\sum_{i=0}^{n+1}\alpha_{i}}.
    \end{align*}
    This completes the induction and establishes the desired equality. 
\end{proof}
\subsection{Finite-Time Analysis of DSPI with Linear Function Approximation}

Observe that, in contrast to the tabular setting where the natural gradient can be expressed in terms of the $Q$-function, Algorithm~\ref{alg:DSPI-LFA} uses only an approximation $W_k^*$ to the $Q$-function at iteration $k$. We therefore write $W_k^* = Q^{\pi_k} + \epsilon_k$, where $\epsilon_k := W_k^* - Q^{\pi_k}$ captures the function approximation error. To separate approximation error from convergence, we assume there exists $\epsilon > 0$ such that $\sup_k \|\epsilon_k\|_\infty \leq \epsilon$, where $\epsilon$ is analogous to the approximation error bound $\epsilon_{\text{approx}}$ defined in \cite{agarwal2021theory}. 

In model-free RL, the approximation $W_k^*$ would need to be estimated via temporal-difference learning with linear function approximation \cite{tsitsiklis1997analysis,srikant2019finite,bhandari2018finite}. This is beyond the scope of this work, as we focus on the policy optimization dynamics.

For simplicity, we consider only a constant stepsize in Algorithm~\ref{alg:DSPI-LFA}. We now state the main convergence result, whose proof is presented in Appendix \ref{subsubsec:pf:thm:DSPI-LFA}.

\begin{theorem}\label{thm:DSPI-LFA}
    Consider $\{\pi_k\}$ generated by Algorithm~\ref{alg:DSPI-LFA}. When choosing $\beta_0=1$ and $\beta_k\equiv \beta\in(0,1]$ for all $k\geq 1$, we have
    \begin{align*}
    \|V^* - V^{\pi_{k}}\|_{\infty} 
    \leq (1-\beta(1-\gamma))^{k-1}
    \left(\gamma \|V^* - V^{\pi_0}\|_{\infty} + \tau \nu_{\max}\right)
    + \frac{2\epsilon}{\beta(1-\gamma)^2},\quad \forall\,k\geq 1.
\end{align*}
\end{theorem}

Theorem \ref{thm:DSPI-LFA} shows that Algorithm \ref{alg:DSPI-LFA} converges at a geometric rate up to a term that is proportional to the function approximation error $\epsilon$.

\subsection{Finite-Time Analysis of NPG with Linear Function Approximation}

Combining Proposition \ref{prop:equivalence-LFA} with Theorem \ref{thm:DSPI-LFA}, we have the following result for NPG with linear function approximation.
\begin{theorem}\label{thm:NPG-LFA}
    Consider $\{\pi_k\}$ generated by Algorithm~\ref{alg:Q-NPG}.  
    When choosing $\alpha_0=\log(m)$ and $\alpha_{k} = \beta\alpha_0/(1-\beta)^{k}$ for all $k\geq 1$ and for any $\beta\in (0,1)$ , we have 
    \begin{align*}
    \|V^* - V^{\pi_{k}}\|_{\infty} 
    \leq (1-\beta(1-\gamma))^{k-1}
    \left(\gamma \|V^* - V^{\pi_0}\|_{\infty} + 1\right)
    + \frac{2\epsilon}{\beta(1-\gamma)^2},\quad \forall\,k\geq 1.
\end{align*}
As a result, by choosing $\beta = 1/2$, for any $\xi>0$, Algorithm~\ref{alg:Q-NPG} achieves 
$\|V^* - V^{\pi_{k}}\|_{\infty}\leq \xi+4\epsilon/(1-\gamma)^2$ with iteration complexity 
$\mathcal{O}((1-\gamma)^{-1}\log(\xi^{-1}(1-\gamma)^{-1}))$.
\end{theorem}
Similar to the tabular setting, the proof of the above theorem follows from the equivalence of DSPI and NPG under a specific choice of stepsize and regularizer. Following the stepsize calculations in Appendix~\ref{pf:thm:NPG}, Proposition~\ref{prop:equivalence-LFA} implies that Algorithm~\ref{alg:Q-NPG} is a special case of Algorithm~\ref{alg:DSPI-LFA}. The convergence rate in Theorem~\ref{thm:NPG-LFA} then follows by taking $\nu(\cdot) = h(\cdot)$ and $\tau = 1/\log(m)$. Since the proof is similar to that of Theorem~\ref{thm:NPG}, we omit the details.

The NPG method with linear function approximation, also referred to as NPG with a log-linear policy class, has been extensively analyzed from a continuous optimization viewpoint \cite{agarwal2021theory,yuan2023linear,alfano2022linear,alfano2023novel,cayci2024convergence}, establishing convergence to the global optimum up to a function approximation error. In particular, \cite{agarwal2021theory} establish $\mathcal{O}(1/\sqrt{k})$ convergence for the best iterate (up to iteration $k$) using a decaying stepsize, while \cite{yuan2023linear} and \cite{alfano2022linear} establish geometric convergence of the last iterate with geometrically increasing stepsize sequences. Considering a more general framework of policy mirror descent, \cite{alfano2023novel} achieve geometric convergence with geometrically increasing stepsize sequence. \cite{cayci2024convergence} show that geometric convergence of the last iterate can be achieved with a constant stepsize, but using entropy regularization of the MDP. The convergence rates in the aforementioned works depend on a concentrability or distribution mismatch coefficient; moreover, the results of \cite{agarwal2021theory,yuan2023linear,alfano2022linear} require a bounded relative condition number associated with the feature matrix. In contrast, Theorem~\ref{thm:NPG-LFA} establishes distribution-free global geometric convergence of the last iterate up to a function approximation error, without requiring additional assumptions on the feature matrix or regularization of the MDP. 

\subsection{Proof of Theorem~\ref{thm:DSPI-LFA}}\label{subsubsec:pf:thm:DSPI-LFA}
The proof follows the same roadmap as in the tabular setting. Specifically, we establish an almost monotonic improvement and an almost contractive recursion, where the ``almost'' arises from the function approximation error.

The proof of the following lemma is given in Appendix~\ref{pf:le:monotone-LFA}. 
\begin{lemma}[Almost Monotonic Improvement]\label{le:monotone-LFA}
    Algorithm~\ref{alg:DSPI-LFA} satisfies 
    \begin{align*}
        Q^{\pi_k} \leq Q^{\pi_{k+1}} + \frac{2\gamma\epsilon}{1-\gamma}\mathbf{1},\quad \forall\,k\geq 0.
    \end{align*}
\end{lemma}
As a corollary of the above lemma, we have the following result on the relation between the averaged iterates $\Bar{W}_{k}$ and the $Q$-functions $Q^{\pi_k}$, which is proved in Appendix \ref{pf:le:coro:Wk-bar_Q-pik}.  
\begin{corollary}\label{coro:Wk-bar_Q-pik}
    Algorithm \ref{alg:DSPI-LFA} satisfies 
    \begin{align*}
        \Bar{W}_{k} \leq Q^{\pi_k} 
    + \delta_k\mathbf{1},\quad \forall\, k \geq 1,
    \end{align*}
    where $\delta_k:=\left(\frac{1+\gamma}{1-\gamma}\right)
    \epsilon \sum_{i=0}^{k-1}(1-\beta)^i$.
\end{corollary}
The above two results can now be used to obtain an almost one-step contraction inequality, which is the linear function approximation analogue of Lemma~\ref{le:contraction_recursion}. We state the result below and prove it in Appendix~\ref{pf:le:contraction_recursion-LFA}.
\begin{lemma}[An Almost One-step Contractive Recursion]\label{le:contraction_recursion-LFA}
    Algorithm \ref{alg:DSPI-LFA} satisfies 
    \begin{align*}
        \|Q^{*} - \Bar{W}_{k+1}\|_{\infty} \leq  (1 - \beta(1-\gamma))\|Q^{*} - \Bar{W}_{k}\|_{\infty} + \beta \gamma\nu_{\max}\eta_{k-1} + \beta \gamma \delta_k + \beta \epsilon,\quad \forall\,k\geq 1.
    \end{align*}
\end{lemma}
The following result translates the above bound on the $Q$-functions into a bound on the value functions. It is analogous to Lemma~\ref{le:translation} and is proved in Appendix~\ref{pf:le:translation-LFA}.
\begin{lemma}\label{le:translation-LFA}
    Algorithm \ref{alg:DSPI-LFA} satisfies 
    \begin{enumerate}[(1)]
        \item $\|V^* - V^{\pi_{k+1}}\|_{\infty} \leq \|Q^* - \Bar{W}_{k+1}\|_{\infty} + \eta_{k}\nu_{\max} + \delta_{k+1}$, for every $k \geq 0$;
        \item $\|Q^{*} - \Bar{W}_{1}\|_{\infty} \leq \gamma \|V^* -  V^{\pi_0}\|_{\infty} + \epsilon$.
    \end{enumerate}
\end{lemma}
The next step toward the final bound is to apply Lemma \ref{le:translation-LFA} (1) and recursively use the inequality in Lemma \ref{le:contraction_recursion-LFA}. Specifically, we obtain
\begin{align*}
    \|V^* - V^{\pi_{k+1}}\|_{\infty} 
    \leq\;& \|Q^* - \Bar{W}_{k+1}\|_{\infty} + \eta_k \nu_{\max} + \delta_{k+1} \\
    \leq\;& (1 - \beta(1-\gamma))\|Q^{*} - \Bar{W}_{k}\|_{\infty} 
    + \beta \left(\gamma \eta_{k-1}\nu_{\max} + \gamma \delta_k + \epsilon \right) + \eta_k \nu_{\max} + \delta_{k+1} \\
    \leq\;& (1-\beta(1-\gamma))^k \|Q^{*} - \Bar{W}_{1}\|_{\infty} \\
    &\; + \sum_{j=1}^{k}(1-\beta(1-\gamma))^{k-j}\beta
    \left(\gamma \eta_{j-1}\nu_{\max} + \gamma \delta_j + \epsilon \right) + \eta_k \nu_{\max} + \delta_{k+1}.
\end{align*}
Applying Lemma \ref{le:translation-LFA} (2) gives
\begin{align}\label{eq:final_recursion-LFA}
    \|V^* - V^{\pi_{k+1}}\|_{\infty} 
    \leq \;& \gamma(1-\beta(1-\gamma))^k
    \|V^* - V^{\pi_0}\|_{\infty}\nonumber\\
    & + \epsilon\underbrace{\left((1-\beta(1-\gamma))^k+\sum_{j=1}^{k}(1-\beta(1-\gamma))^{k-j}\beta\right)}_{:=O_1}\nonumber\\
    & + \nu_{\max} \underbrace{\left(\sum_{j=1}^{k}(1-\beta(1-\gamma))^{k-j}\beta\gamma\eta_{j-1}+\eta_k\right)}_{:=O_2}\nonumber\\
    &+ \underbrace{\sum_{j=1}^{k}(1-\beta(1-\gamma))^{k-j}\beta
    \gamma \delta_j +\delta_{k+1}}_{:=O_3}.
\end{align}
For the term $O_1$, we have
\begin{align*}
    O_1=(1-\beta(1-\gamma))^k+\frac{1}{1-\gamma}\left[1-(1-\beta(1-\gamma))^{k}\right]
    \leq \frac{1}{1-\gamma}.
\end{align*}
For the term $O_2$, since $\eta_k=\tau (1-\beta)^k$, we have
\begin{align*}
    O_2= \,&\tau\beta\gamma\sum_{j=1}^{k}(1-\beta(1-\gamma))^{k-j}(1-\beta)^{j-1}+\tau(1-\beta)^k\\
    = \,&\tau\beta\gamma(1-\beta(1-\gamma))^{k-1}\sum_{j=1}^{k}\left(\frac{1-\beta}{1-\beta(1-\gamma)}\right)^{j-1}+\tau(1-\beta)^k\\
    = \,&\tau\beta\gamma(1-\beta(1-\gamma))^{k-1}\frac{1-\left(\frac{1-\beta}{1-\beta(1-\gamma)}\right)^{k}}{1-\left(\frac{1-\beta}{1-\beta(1-\gamma)}\right)}+\tau(1-\beta)^k\\
=\;&\tau\left((1-\beta(1-\gamma))^{k}-(1-\beta)^{k}\right)+\tau(1-\beta)^k\\
= \,&\tau(1-\beta(1-\gamma))^k.
\end{align*}
For the term $O_3$, we have
\begin{align*}
    &\sum_{j=1}^{k}(1-\beta(1-\gamma))^{k-j}\beta
    \gamma \delta_j +\delta_{k+1}\\
    =\,&
    \gamma \left(\frac{1+\gamma}{1-\gamma}\right)\epsilon\sum_{j=1}^{k}(1-\beta(1-\gamma))^{k-j}(1-(1-\beta)^j)+\frac{1}{\beta}\left(\frac{1+\gamma}{1-\gamma}\right)(1-(1-\beta)^{k+1})
    \epsilon \\
    \leq \,&\frac{1+\gamma}{\beta(1-\gamma)^2}\epsilon.
\end{align*}
Substituting the bounds obtained for $O_1$, $O_2$, and $O_3$ into \eqref{eq:final_recursion-LFA}, we have
\begin{align*}
    \|V^* - V^{\pi_{k+1}}\|_{\infty} 
    \leq (1-\beta(1-\gamma))^k
    \left(\gamma \|V^* - V^{\pi_0}\|_{\infty} + \tau \nu_{\max}\right)
    + \frac{2\epsilon}{\beta(1-\gamma)^2}.
\end{align*}

\subsection{Proof of All Technical Lemmas}
\subsubsection{Proof of Lemma~\ref{le:monotone-LFA}}\label{pf:le:monotone-LFA}
It is enough to show, for any $k \geq 0$, that $Q^{\pi_k} \leq \mathcal{H}^{\pi_{k+1}}(Q^{\pi_k}) + 2 \gamma \epsilon \mathbf{1}$. Once this is established, repeatedly applying $\mathcal{H}^{\pi_{k+1}}$ on both sides gives us $Q^{\pi_k} \leq [\mathcal{H}^{\pi_{k+1}}]^n(Q^{\pi_k}) + 2\epsilon \mathbf{1}\sum_{i=1}^n\gamma^i$. Letting $n\rightarrow\infty$ gives us the desired result.

We next prove that $Q^{\pi_k} \leq \mathcal{H}^{\pi_{k+1}}(Q^{\pi_k}) + 2\gamma \epsilon \mathbf{1}$ for all $k \geq 0$. Similar to the proof of Lemma \ref{le:monotone}, we consider the two cases $k = 0$ and $k \geq 1$ separately.
\begin{enumerate}[(1)]
    \item For $k = 0$, since $\beta_0 = 1$, we have $\Bar{W}_{1} = W_0^*$. It follows that
    \begin{align*}
        Q^{\pi_0} &= \mathcal{H}^{\pi_0}(Q^{\pi_0}) \\
        & = \mathcal{H}^{\pi_0}(W_{0}^* - \epsilon_{0}) \\
        & \leq  \mathcal{H}^{\pi_0}(W_{0}^* +\epsilon \mathbf{1}) \\
        & = \mathcal{H}^{\pi_0}(W_{0}^*) + \gamma \epsilon \mathbf{1} \tag{Translation invariance}\\
        & = \mathcal{H}^{\pi_0}(\Bar{W}_{1}) + \gamma \epsilon
        \mathbf{1} \\
        & = \mathcal{H}_{\eta_0}^{\pi_0}(\Bar{W}_{1}) - \eta_0 f(\pi_0) + \gamma \epsilon \mathbf{1} \\
        & \leq  \mathcal{H}_{\eta_0}^{\pi_1}(\Bar{W}_{1}) - \eta_0 f(\pi_0) + \gamma \epsilon\mathbf{1} \tag{Algorithm \ref{alg:DSPI-LFA}, Line 5}\\
        & = \mathcal{H}^{\pi_1}(\Bar{W}_{1}) +\eta_0 [f(\pi_1) - f(\pi_0)] + \gamma \epsilon\mathbf{1} \\
        & \leq \mathcal{H}^{\pi_1}(Q^{\pi_0} + \epsilon_{0}) + \gamma \epsilon\mathbf{1} \tag{Algorithm \ref{alg:DSPI-LFA}, Line 1}\\
        & \leq \mathcal{H}^{\pi_1}(Q^{\pi_0} + \epsilon \mathbf{1}) + \gamma \epsilon\mathbf{1}\\
        & \leq \mathcal{H}^{\pi_1}(Q^{\pi_0}) + 2 \gamma \epsilon\mathbf{1}.\tag{Translation invariance}
    \end{align*}
    \item For $k \geq 1$, observe that
    \begin{align*}
        Q^{\pi_k} & = \mathcal{H}^{\pi_k}(Q^{\pi_k}) \\
        & = \mathcal{H}^{\pi_k}(W_{k}^* - \epsilon_{k}) \\
        & \leq \mathcal{H}^{\pi_k}(W_{k}^* +\epsilon\mathbf{1}) \\
        & \leq \mathcal{H}^{\pi_k}(W_{k}^*) + \gamma \epsilon \mathbf{1} \tag{Translation invariance}\\
        & = \mathcal{H}^{\pi_k}\left(\frac{\Bar{W}_{k+1} - (1-\beta_k)\Bar{W}_{k}}{\beta_k}\right) + \gamma \epsilon \mathbf{1} \tag{Algorithm \ref{alg:DSPI-LFA}, Line 4}\\
        & = \mathcal{H}^{\pi_k}_{\eta_{k-1}}\left(\frac{\Bar{W}_{k+1} - (1-\beta_k)\Bar{W}_{k}}{\beta_k}\right) - \eta_{k-1}f(\pi_k) + \gamma \epsilon \mathbf{1} \\
        & = \frac{1}{\beta_k}\mathcal{H}^{\pi_k}_{\eta_{k-1}}(\Bar{W}_{k+1}) - \frac{1-\beta_k}{\beta_k} \mathcal{H}^{\pi_k}_{\eta_{k-1}}(\Bar{W}_{k}) - \eta_{k-1}f(\pi_k) + \gamma \epsilon \mathbf{1}\tag{$\mathcal{H}^{\pi_k}_{\eta_{k-1}}$ being affine}\\
        & = \frac{1}{\beta_k}\mathcal{H}^{\pi_k}_{\eta_{k}}(\Bar{W}_{k+1}) - \frac{1-\beta_k}{\beta_k} \mathcal{H}^{\pi_k}_{\eta_{k-1}}(\Bar{W}_{k}) -\frac{\eta_{k}-(1-\beta_k)\eta_{k-1}}{\beta_k}f(\pi_k)+ \gamma \epsilon \mathbf{1}\\
        & = \frac{1}{\beta_k}\mathcal{H}^{\pi_k}_{\eta_{k}}(\Bar{W}_{k+1}) - \frac{1-\beta_k}{\beta_k} \mathcal{H}^{\pi_k}_{\eta_{k-1}}(\Bar{W}_{k}) + \gamma \epsilon \mathbf{1},
    \end{align*}
    where the last step follows from $\eta_{k}=(1-\beta_k)\eta_{k-1}$. 
    
   Since Line 5 of Algorithm \ref{alg:DSPI-LFA} implies
    $\mathcal{H}_{\eta_k}^{\pi_{k+1}}(\Bar{W}_{k+1}) = \mathcal{H}_{\eta_{k}}(\Bar{W}_{k+1})$ for all $k \geq 0$, we have
    \begin{align*}
        \mathcal{H}^{\pi_k}_{\eta_{k}}(\Bar{W}_{k+1}) & \leq \mathcal{H}_{\eta_{k}}(\Bar{W}_{k+1}) = \mathcal{H}^{\pi_{k+1}}_{\eta_{k}}(\Bar{W}_{k+1}) \\
        \mathcal{H}^{\pi_k}_{\eta_{k-1}}(\Bar{W}_{k}) & = \mathcal{H}_{\eta_{k-1}}(\Bar{W}_{k}) \geq \mathcal{H}^{\pi_{k+1}}_{\eta_{k-1}}(\Bar{W}_{k}).
    \end{align*}
    Combining the previous two inequalities, we obtain
    \begin{align*}
        Q^{\pi_k} & \leq \frac{1}{\beta_k} \mathcal{H}^{\pi_{k+1}}_{\eta_{k}}(\Bar{W}_{k+1}) - \frac{1-\beta_k}{\beta_k} \mathcal{H}^{\pi_{k+1}}_{\eta_{k-1}}(\Bar{W}_{k}) + \gamma \epsilon \mathbf{1} \\
        & = \frac{1}{\beta_k} \mathcal{H}^{\pi_{k+1}}(\Bar{W}_{k+1}) - \frac{1-\beta_k}{\beta_k} \mathcal{H}^{\pi_{k+1}}(\Bar{W}_{k}) + \frac{\eta_k - (1-\beta_k)\eta_{k-1}}{\beta_k}f(\pi_{k+1}) + \gamma \epsilon \mathbf{1} \\
        & = \frac{1}{\beta_k} \mathcal{H}^{\pi_{k+1}}(\Bar{W}_{k+1}) - \frac{1-\beta_k}{\beta_k} \mathcal{H}^{\pi_{k+1}}(\Bar{W}_{k}) + \gamma \epsilon\mathbf{1} \tag{$\eta_k = (1-\beta_k)\eta_{k-1}$}\\
        & = \mathcal{H}^{\pi_{k+1}}\left(\frac{\Bar{W}_{k+1} - (1-\beta_k)\Bar{W}_{k}}{\beta_k}\right) + \gamma \epsilon \mathbf{1} \tag{$\mathcal{H}^{\pi_{k+1}}$ being affine}\\
        & = \mathcal{H}^{\pi_{k+1}}(W_k^*) + \gamma \epsilon\mathbf{1} \tag{Algorithm \ref{alg:DSPI-LFA}, Line 4}\\
        & = \mathcal{H}^{\pi_{k+1}}(Q^{\pi_k} + \epsilon_k) + \gamma \epsilon \mathbf{1} \\
        & \leq  \mathcal{H}^{\pi_{k+1}}(Q^{\pi_k} + \epsilon\mathbf{1}) + \gamma \epsilon \mathbf{1} \\
        & \leq \mathcal{H}^{\pi_{k+1}}(Q^{\pi_k}) + 2 \gamma \epsilon \mathbf{1}\tag{Translation invariance}
    \end{align*}
\end{enumerate}

\subsubsection{Proof of Corollary \ref{coro:Wk-bar_Q-pik}}\label{pf:le:coro:Wk-bar_Q-pik}
The proof is by induction. Since $\bar{W}_{0}=0$ and $\beta_0=1$, we have by Lemma \ref{le:monotone-LFA} that
\begin{align*}
    \bar{W}_{1}
    = W_0^*
    = Q^{\pi_0}+\epsilon_0
    \leq Q^{\pi_0}+\epsilon\mathbf{1}
    \leq Q^{\pi_1}+\frac{2\gamma\epsilon}{1-\gamma}\mathbf{1}+\epsilon\mathbf{1}\leq Q^{\pi_1}+\frac{\epsilon(1+\gamma)}{1-\gamma}\mathbf{1}.
\end{align*}
Thus, the base case holds.

Now suppose that, for some $n\geq 1$,
\begin{align*}
    \bar{W}_{n}
    \leq Q^{\pi_n}
    + \left(\frac{1+\gamma}{1-\gamma}\right)
    \epsilon \sum_{i=0}^{n-1}(1-\beta)^i\mathbf{1}.
\end{align*}
Then,
\begin{align*}
    \bar{W}_{n+1}
    &= (1-\beta)\bar{W}_{n}+\beta W_n^* \\
    &\leq (1-\beta)Q^{\pi_n}
    + (1-\beta)\left(\frac{1+\gamma}{1-\gamma}\right)
    \epsilon \sum_{i=0}^{n-1}(1-\beta)^i\mathbf{1}
    + \beta(Q^{\pi_n}+\epsilon_n) \tag{The induction hypothesis}\\
    &\leq Q^{\pi_n}
    + \left(\frac{1+\gamma}{1-\gamma}\right)
    \epsilon \sum_{i=1}^{n}(1-\beta)^i\mathbf{1}
    + \beta\epsilon\mathbf{1} \\
    &\leq Q^{\pi_{n+1}}
    + \frac{2\gamma\epsilon}{1-\gamma}\mathbf{1}
    + \left(\frac{1+\gamma}{1-\gamma}\right)
    \epsilon \sum_{i=1}^{n}(1-\beta)^i\mathbf{1}
    + \beta\epsilon\mathbf{1} \tag{Lemma \ref{le:monotone-LFA}}\\
    &\leq Q^{\pi_{n+1}}
    + \left(\frac{1+\gamma}{1-\gamma}\right)
    \epsilon \sum_{i=0}^{n}(1-\beta)^i\mathbf{1}.
\end{align*}
This completes the induction.

\subsubsection{Proof of Lemma \ref{le:contraction_recursion-LFA}}\label{pf:le:contraction_recursion-LFA}
By Line 4 of Algorithm \ref{alg:DSPI-LFA}, we have 
\begin{align}
    \|Q^{*} - \Bar{W}_{k+1}\|_{\infty}&=\|(1 - \beta)(Q^{*} - \Bar{W}_{k}) + \beta (Q^{*} - W_{k}^*) \|_\infty\nonumber\\
    & \leq (1-\beta) \|Q^{*} - \Bar{W}_{k}\|_{\infty} + \beta \|Q^{*} - W_{k}^*\|_{\infty} \nonumber\\
    & = (1-\beta) \|Q^{*} - \Bar{W}_{k}\|_{\infty} + \beta \|Q^{*} - Q^{\pi_k} - \epsilon_{k}\|_{\infty} \nonumber\\
    & \leq (1-\beta) \|Q^{*} - \Bar{W}_{k}\|_{\infty} + \beta \|Q^{*} - Q^{\pi_k}\|_{\infty} + \beta \epsilon.\label{bound:Q*W}
\end{align}
We next bound $\|Q^{*} - Q^{\pi_k}\|_\infty$ in terms of $\|Q^{*} - \Bar{W}_{k}\|_\infty$. Observe that
\begin{align*}
    Q^* - Q^{\pi_k} & = \mathcal{H}(Q^{*}) - \mathcal{H}^{\pi_k}(Q^{\pi_k}) \\
    & \leq \mathcal{H}(Q^{*}) - \mathcal{H}^{\pi_k}(\Bar{W}_{k} - \delta_k \mathbf{1}) \tag{Corollary \ref{coro:Wk-bar_Q-pik}} \\
    & = \mathcal{H}(Q^{*}) - \mathcal{H}^{\pi_k}(\Bar{W}_{k}) + \gamma \delta_k \mathbf{1} \tag{Translation invariance}\\
    & = \mathcal{H}(Q^{*}) - \mathcal{H}^{\pi_k}_{\eta_{k-1}}(\Bar{W}_{k}) + \eta_{k-1}f(\pi_{k}) + \gamma \delta_k \mathbf{1} \\
    & = \mathcal{H}(Q^{*}) - \mathcal{H}_{\eta_{k-1}} (\Bar{W}_{k}) + \eta_{k-1}f(\pi_{k}) + \gamma \delta_k \mathbf{1} \tag{Algorithm~\ref{alg:DSPI-LFA}, Line 5}\\
    & \leq \mathcal{H}(Q^{*}) - \mathcal{H} (\Bar{W}_{k}) + \eta_{k-1}f(\pi_{k}) + \gamma \delta_k \mathbf{1} \tag{$\nu(\cdot)\geq 0$}\\
    & \leq \|\mathcal{H}(Q^{*}) - \mathcal{H} (\Bar{W}_{k})\|_\infty \mathbf{1} + \eta_{k-1}\gamma \nu_{\max} \mathbf{1}+ \gamma \delta_k \mathbf{1}\\
    \leq \,&\gamma\|Q^{*}- \Bar{W}_{k}\|_\infty \mathbf{1} + \eta_{k-1}\gamma \nu_{\max} \mathbf{1}+ \gamma \delta_k \mathbf{1}.
\end{align*}
Since $Q^{*} - Q^{\pi_k} \geq 0$, we obtain
\begin{align*}
    \|Q^* - Q^{\pi_k}\|_{\infty}  \leq \gamma\|Q^{*} - \Bar{W}_{k}\|_{\infty} + \gamma\nu_{\max}\eta_{k-1} + \gamma \delta_k. 
\end{align*}
Substituting the above bound into \eqref{bound:Q*W} yields 
\begin{align*}
    \|Q^{*} - \Bar{W}_{k+1}\|_{\infty} & \leq (1-\beta) \|Q^{*} - \Bar{W}_{k}\|_{\infty} + \beta \gamma\|Q^{*} - \Bar{W}_{k}\|_{\infty} + \beta \gamma\nu_{\max}\eta_{k-1} + \beta \gamma \delta_k + \beta \epsilon \\
    & = (1 - \beta(1-\gamma))\|Q^{*} - \Bar{W}_{k}\|_{\infty} + \beta \gamma\nu_{\max}\eta_{k-1} + \beta \gamma \delta_k + \beta \epsilon.
\end{align*}

\subsubsection{Proof of Lemma \ref{le:translation-LFA}}\label{pf:le:translation-LFA}
\begin{enumerate}[(1)]
    \item For any $k \geq 0$ and $s \in \mathcal{S}$, 
    \begin{align*}
        V^*(s) - V^{\pi_{k+1}}(s) & = \max_{a\in\mathcal{A}}Q^*(s,a) - \pi_{k+1}(s)^{\top} Q^{\pi_{k+1}}(s) \\
        & \leq \max_{a\in\mathcal{A}}Q^*(s,a) - \pi_{k+1}(s)^{\top} (\Bar{W}_{k+1}(s) - \delta_{k+1}\mathbf{1}) \tag{Corollary \ref{coro:Wk-bar_Q-pik}} \\
        & = \max_{a\in\mathcal{A}}Q^*(s,a) - \pi_{k+1}(s)^{\top} \Bar{W}_{k+1}(s) + \delta_{k+1} \\
        & \leq  \max_{a\in\mathcal{A}}Q^*(s,a) - \left\{\pi_{k+1}(s)^{\top} \Bar{W}_{k+1}(s) + \eta_k\nu(\pi_{k+1}(s)) \right\} + \eta_k\nu_{\max} + \delta_{k+1} \\
        & = \max_{\mu_{\Delta(\mathcal{A})}}\mu^{\top}Q^*(s,a) - \max_{\mu\in\Delta(\mathcal{A})}\left\{ \mu^{\top} \Bar{W}_{k+1}(s) + \eta_k\nu(\mu) \right\} + \eta_k\nu_{\max}+ \delta_{k+1} \tag{Algorithm \ref{alg:DSPI-LFA}, Line 5}\\
        & \leq  \max_{\mu_{\Delta(\mathcal{A})}}\mu^{\top}Q^*(s,a) - \max_{\mu\in\Delta(\mathcal{A})}\mu^{\top} \Bar{W}_{k+1}(s)  + \eta_k\nu(\pi_{k+1}(s)) + \delta_{k+1} \tag{$\nu(\cdot)\geq 0$}\\
        & \leq \max_{\mu_{\Delta(\mathcal{A})}} \left\{ \mu^{\top}( Q^*(s,a) - \Bar{W}_{k+1}(s) ) \right\} + \eta_{k}\nu_{\max} + \delta_{k+1} \\
        & \leq \|Q^* - \Bar{W}_{k+1}\|_{\infty} + \nu_{\max}\eta_{k} + \delta_{k+1}
    \end{align*}
    Since $V^*(s) - V^{\pi_{k+1}}(s) \geq 0$ and the above inequality holds for every $s \in \mathcal{S}$, we have
    \begin{align*}
        \|V^* - V^{\pi_{k+1}}\|_{\infty} \leq \|Q^* - \Bar{W}_{k+1}\|_{\infty} + \eta_{k}\nu_{\max} + \delta_{k+1}.
    \end{align*}
    \item Since $\beta_0=1$ and $\Bar{W}_{1} = W_{0}^{*}$, we have 
    \begin{align*}
        \|Q^{*} - \Bar{W}_{1}\|_{\infty}  = \|Q^{*} - W_{0}^{*}\|_{\infty} 
         = \|Q^{*} - Q^{\pi_0} - \epsilon_{0}\|_{\infty} 
         \leq \|Q^{*} - Q^{\pi_0}\|_{\infty} + \epsilon.
    \end{align*}
    Using the Bellman equation for the $Q$-function, we have
    \begin{align*}
        \|Q^{*} - Q^{\pi_0}\|_{\infty} & = \|\mathcal{H}(Q^*) - \mathcal{H}^{\pi_0}(Q^{\pi_0})\|_{\infty} \\
        & = \gamma\max_{(s,a)}\left|\sum_{s'\in\mathcal{S}}p(s'|s,a) \left( \max_{a'\in\mathcal{A}}Q^*(s',a') - \sum_{a'\in\mathcal{A}}\pi_{0}(a'|s') Q^{\pi_0}(s',a') \right)\right| \\
        & = \gamma\max_{(s,a)}\left|\sum_{s'\in\mathcal{S}}p(s'|s,a) \left( V^*(s') - V^{\pi_0}(s') \right)\right| \\
        & \leq \gamma \|V^* -  V^{\pi_0}\|_{\infty}.
    \end{align*}
    Combining the previous two inequalities yields
    \begin{align*}
        \|Q^{*} - \Bar{W}_{1}\|_{\infty} \leq \gamma \|V^* -  V^{\pi_0}\|_{\infty} + \epsilon.
    \end{align*}
\end{enumerate}

\section{Natural Policy Gradient for Stochastic Shortest Path Problems}\label{sec:SSP}

Consider an infinite-horizon undiscounted MDP $\mathcal{M} = (\mathcal{S}_0, \mathcal{A}, p, \mathcal{R})$, where $\mathcal{S}_0$ is a finite state space that includes a terminal state, $\mathcal{A}$ is a finite action space, $\mathcal{R} : \mathcal{S}_0 \times \mathcal{A} \to [-1,0]$ is the reward function, and $p$ is the transition kernel. We denote the terminal state by $s_{\text{tm}}$ and define $\mathcal{S} = \mathcal{S}_0 \setminus \{s_{\text{tm}}\}$. Once the system reaches $s_{\text{tm}}$, it remains there and receives no further rewards, i.e., $p(s_{\text{tm}} \mid s_{\text{tm}}, a) = 1$ and $\mathcal{R}(s_{\text{tm}}, a) = 0$ for all $a \in \mathcal{A}$. As in the discounted setting, we denote $n = |\mathcal{S}|$ and $m = |\mathcal{A}|$.

Given a stationary policy $\pi : \mathcal{S} \to \Delta(\mathcal{A})$, the expected total reward starting from a state $s \in \mathcal{S}$ is defined as
\begin{align}\label{def:V_SSP}
    V^{\pi}(s) = \limsup_{K \to \infty}\mathbb{E}_{\pi}\left[\sum_{k=0}^{K} \mathcal{R}(S_{k}, A_{k}) \;\middle|\; S_0 = s \right],
\end{align}
where $A_{k} \sim \pi(\cdot \mid S_{k})$ and $S_{k+1} \sim p(\cdot \mid S_k, A_k)$. 
A stationary policy $\pi^*$ is said to be optimal if $V^*(s) \coloneqq V^{\pi^*}(s) \geq V^{\pi}(s)$ for all stationary policies $\pi$. 

To ensure termination under the optimal policy, the notion of a proper policy is commonly adopted in the stochastic shortest path (SSP) literature \cite{yu2013boundedness,yu2013q,bertsekas1991analysis,chen2022policy}. The formal definition of a proper policy is stated below.
\begin{definition}\label{def:proper-policy}
    A policy $\pi$ is said to be proper if, for any initial state $s \in \mathcal{S}$, there is a positive probability of reaching the terminal state within at most $n$ stages, i.e., $\mathbb{P}\{S_{n} \neq s_{\text{tm}} \mid S_0 = s\} < 1$ for all $s \in \mathcal{S}$. 
\end{definition}
The notion of a proper policy was standardized in the SSP literature following \cite{bertsekas1991analysis}, where it was used to relax the prevailing assumption that the reward function be either nonnegative or nonpositive. In particular, \cite{bertsekas1991analysis} required that there exists at least one proper policy, and for every improper policy $\pi$, there exists a state $s \in \mathcal{S}$ such that $V^{\pi}(s) = -\infty$, and established asymptotic convergence of VI and PI. Under the same assumption, subsequent works have also established asymptotic convergence of Q-learning \cite{tsitsiklis1994asynchronous,yu2013boundedness} and of variants of PI and Q-learning \cite{yu2013q}. However, this assumption may not be sufficient for establishing convergence rates, which typically relies on a contraction property of the Bellman optimality operator\cite{bertsekas1996neuro,yu2013q}. In SSP setting, the Bellman optimality operator is only non-expansive with respect to the $\ell_\infty$ norm. Nonetheless, it becomes a contraction with respect to a weighted $\ell_\infty$ norm (cf. Lemma~\ref{le:weighted_sup_norm}) when all stationary policies are proper--a standard condition in the literature \cite{bertsekas1991analysis,bertsekas1996neuro,yu2013q} when establishing contraction properties of the Bellman operator. We therefore adopt the same assumption, formally stated below. 
\begin{assumption}\label{as:proper-all}
    All stationary policies are proper.
\end{assumption}

For a proper policy, since the MDP terminates in finite time, the limit in \eqref{def:V_SSP} exists. 
Under Assumption \ref{as:proper-all}, it has been established in \cite{bertsekas1991analysis} that 
there exists an optimal proper policy $\pi^*$ for which $V^*(s)$ is finite for every $s \in \mathcal{S}$. 
Furthermore, $V^*$ satisfies the Bellman equation 
\begin{align*}
    V^*(s) = \max_{a \in \mathcal{A}} \left\{ \mathcal{R}(s,a) + \sum_{s' \in \mathcal{S}_0} p(s' \mid s,a)\, V^*(s') \right\}.
\end{align*}
For a proper policy $\pi$, the state-action value function $Q^{\pi}$ is analogously defined as 
\begin{align*}
    Q^{\pi}(s,a) = \lim_{K \to \infty}\mathbb{E}\left[\sum_{k=0}^{K} \mathcal{R}(S_k,A_k) \;\middle|\; S_0 = s, A_0 = a \right].
\end{align*}
Under Assumption \ref{as:proper-all}, there exists a unique optimal state-action value function $Q^*$, which is related to the optimal value function via $V^*(s) = \max_{a \in \mathcal{A}} Q^*(s,a)$ and uniquely solves the Bellman equation $Q = \widetilde{\mathcal{H}}(Q)$ \cite{bertsekas1991analysis}, where $\widetilde{\mathcal{H}}:\mathbb{R}^{mn}\to\mathbb{R}^{mn}$ is the Bellman optimality operator defined as 
\begin{align*}
    [\widetilde{\mathcal{H}}(Q)](s,a)
    =\,& \mathcal{R}(s,a) + \sum_{s'\in\mathcal{S}} p(s'\mid s,a) \max_{a'\in\mathcal{A}} Q(s',a'), \quad \forall\,(s,a)\in\mathcal{S}\times\mathcal{A}.
\end{align*}
Similarly, for a proper stationary policy $\pi$, $Q^\pi$ uniquely solves the Bellman equation $Q = \widetilde{\mathcal{H}}^{\pi}(Q)$, where $\widetilde{\mathcal{H}}^\pi:\mathbb{R}^{mn}\to\mathbb{R}^{mn}$ is the Bellman operator associated with $\pi$ defined as 
\begin{align*}
    [\widetilde{\mathcal{H}}^\pi(Q)](s,a)
    =\,& \mathcal{R}(s,a) + \sum_{s'\in\mathcal{S}} p(s'\mid s,a)\sum_{a'\in\mathcal{A}}\pi(a'\mid s') Q(s',a'),\quad \forall\,(s,a)\in\mathcal{S}\times\mathcal{A}. 
\end{align*}
Observe that the sums in the above expressions are over $\mathcal{S} = \mathcal{S}_{0}\setminus\{ s_{\text{tm}} \}$. Further, due to absence of a discount factor, the Bellman optimality operator $\widetilde{\mathcal{H}}$ and the Bellman operator $\widetilde{\mathcal{H}}^{\pi}$ for arbitrary stationary policy $\pi$ are only non-expansive with respect to the $\ell_\infty$-norm. Nonetheless, both $\widetilde{\mathcal{H}}$ and $\widetilde{\mathcal{H}}^\pi$ are monotonic: for $Q_1 \leq Q_2$, we have $\widetilde{\mathcal{H}}(Q_1) \leq \widetilde{\mathcal{H}}(Q_2)$ and $\widetilde{\mathcal{H}}^\pi(Q_1) \leq \widetilde{\mathcal{H}}^\pi(Q_2)$.   

\subsection{Natural Policy Gradient for Stochastic Shortest Path Problems}\label{sec:PDA_SSP}
In comparison to the discounted setting, policy gradient methods in the SSP setting are relatively underexplored. Policy optimization algorithms for SSP problems have been analyzed by \cite{chen2022policy} in terms of regret bounds. Concerning last-iterate convergence, policy gradient methods for undiscounted total reward MDPs were studied by \cite{lee2025policy} only very recently. The NPG update for the SSP problem admits a similar formulation as that for the discounted problem \cite{lee2025policy}; we present it in Algorithm \ref{alg:PDA_SSP}. 
\begin{algorithm}[ht]\caption{Natural Policy Gradient for Stochastic Shortest Path Problems}\label{alg:PDA_SSP}
	\begin{algorithmic}[1]
		\STATE \textbf{Input:} Initialize $\pi_0$ as the uniform policy and $\theta_0=\mathbf{0}\in\mathbb{R}^{mn}$.
		\FOR{$k=0,1,2,3,\cdots$}
        \STATE $\theta_{k+1}=\theta_k+\alpha_k Q^{\pi_k}$   
        \STATE $\pi_{k+1}(s)=\argmax_{\mu\in\Delta(\mathcal{A})} \left\{\mu^\top \theta_{k+1}(s)+h(\mu)\right\}$ for any $s\in\mathcal{S}$.    
		\ENDFOR
	\end{algorithmic}
\end{algorithm} 

\subsection{Reformulation as Doubly Smoothed Policy Iteration}\label{sec:DSPI_SSP}
The development of the DSPI algorithm for SSP problems follows the same roadmap as that for the discounted problems. To that end, we first introduce the smoothed versions of the Bellman operators for the SSP setting. Given a non-negative, bounded, and concave function $\nu:\Delta(\mathcal{A})\to \mathbb{R}$ and $\eta \geq 0$, the smoothed Bellman optimality operator $\widetilde{\mathcal{H}}_\eta:\mathbb{R}^{mn}\to\mathbb{R}^{mn}$ is defined as 
\begin{align*}
    [\widetilde{\mathcal{H}}_\eta(Q)](s,a)
    =\,& \mathcal{R}(s,a) + \sum_{s'\in\mathcal{S}} p(s'\mid s,a)\max_{\mu\in\Delta(\mathcal{A})}\left\{ \mu^\top Q(s') + \eta \nu(\mu) \right\}, 
    \quad \forall\,(s,a)\in\mathcal{S}\times\mathcal{A}. 
\end{align*}
Similarly, given policy $\pi$, the smoothed Bellman operator $\widetilde{\mathcal{H}}^\pi_\eta:\mathbb{R}^{mn}\to\mathbb{R}^{mn}$ associated with $\pi$ is defined as 
\begin{align*}
    [\widetilde{\mathcal{H}}^\pi_\eta(Q)](s,a)
    =\,& \mathcal{R}(s,a) + \sum_{s'\in\mathcal{S}} p(s'\mid s,a)\left\{ \pi(s')^\top Q(s') + \eta \nu(\pi(s')) \right\}, 
    \quad \forall\,(s,a)\in\mathcal{S}\times\mathcal{A}.
\end{align*}
Similar to the discounted setting, the smoothed Bellman operator $\widetilde{\mathcal{H}}^\pi_\eta$ is related to $\widetilde{\mathcal{H}}^\pi$ as $\widetilde{\mathcal{H}}^{\pi}_{\eta}(Q) = \widetilde{\mathcal{H}}^{\pi}(Q) + \eta g(\pi)$, where the operator $g:\mathbb{R}^{mn}\to\mathbb{R}^{mn}$ is defined as 
\begin{align*}
    [g(\pi)](s,a) = \sum_{s'\in\mathcal{S}} p(s'\mid s, a)\; \nu(\pi(s'))
\end{align*}
for all $(s,a)\in\mathcal{S}\times\mathcal{A}$. Moreover, the operators $\widetilde{\mathcal{H}}_\eta$ and $\widetilde{\mathcal{H}}^\pi_\eta$ for arbitrary stationary policy $\pi$ and for $\eta \geq 0$ are non-expansive with respect to the $\ell_\infty$ norm and are monotonic. These properties can be established using arguments similar to those in Appendix~\ref{ap:smooth_Bellman_properties} and hence are omitted. 

With the smoothed versions of the Bellman operators defined, we can now formulate the DSPI algorithm for the SSP problem, which differs from that in the discounted setting only with respect to the Bellman operators. We state the algorithm below. 
\begin{algorithm}[ht]\caption{Doubly Smoothed Policy Iteration for Stochastic Shortest Path Problems}\label{alg:DSPI_SSP}
	\begin{algorithmic}[1]
		\STATE \textbf{Input:} Initialize $\pi_0(s)\in \argmax_{\mu\in\Delta(\mathcal{A})}\nu(\mu)$ for all $s\in\mathcal{S}$ and $\bar{Q}_0=\mathbf{0}$.
		\FOR{$k=0,1,2,3,\cdots$}
        \STATE Maintain a running average of past $Q$-functions: $\bar{Q}_{k+1}=(1-\beta_k)\bar{Q}_k+\beta_k Q^{\pi_k}$,
        where $\beta_k\in (0,1]$ is the stepsize.
        \STATE Choose $\pi_{k+1}$ such that  $\widetilde{\mathcal{H}}_{\eta_k}^{\pi_{k+1}}(\bar{Q}_{k+1})=\widetilde{\mathcal{H}}_{\eta_k}(\bar{Q}_{k+1})$,
        where $\eta_k=\tau \prod_{j=1}^k(1-\beta_j)$ with $\tau\geq 0$ being a tunable parameter.
		\ENDFOR
	\end{algorithmic}
\end{algorithm}

Observe that the policy update is now expressed in terms of the SSP analogue of the smoothed Bellman operators. A concrete way to implement the update step remains the same as that in the discounted setting:
\begin{align}\label{eq:DSPI_SSP_explicit}
        \pi_{k+1}(s)\in \argmax_{\mu \in \Delta(\mathcal{A})} \Big\{
        \mu^\top \bar{Q}_{k+1}(s)
        +\tau \prod_{j=1}^k(1-\beta_j)\nu(\mu)
    \Big\},\quad \forall\,s\in\mathcal{S}.
\end{align}
However, as discussed in Section \ref{sec:DSPI}, the Bellman operator viewpoint facilitates a simplified analysis and leads to improved convergence rates, leveraging the properties of the Bellman operators. 

The DSPI algorithm includes the NPG algorithm as a special case, as stated below.
\begin{proposition}\label{prop:equivalence_SSP}
    In Algorithm \ref{alg:DSPI_SSP}, when choosing $\beta_k = \alpha_k / \sum_{i=0}^k \alpha_i$, $\tau=1/\alpha_0$, and $\nu(\cdot)=h(\cdot)$, the sequence of policies generated by Algorithm~\ref{alg:DSPI_SSP} and Algorithm \ref{alg:PDA_SSP} are identical.
\end{proposition}
Since the NPG and DSPI updates for the SSP problems remain the same as those for the discounted setting, the proof of the above result remains the same as that of Proposition \ref{prop:equivalence}, and hence we omit it here.  

\subsection{Finite-Time Analysis of DSPI for Stochastic Shortest Path Problems}
Following is the main result showing convergence of Algorithm~\ref{alg:DSPI_SSP}. For simplicity, we state it only for using a constant stepsize. 
\begin{theorem}\label{thm:DSPI-SSP}
    Consider the sequence $\{\pi_k\}$ generated by Algorithm~\ref{alg:DSPI_SSP}. Let the stepsizes be chosen as $\beta_0 = 1$ and $\beta_k \equiv \beta \in (0,1]$ for all $k \ge 1$. Then, under Assumption \ref{as:proper-all}, there exists a $\kappa\in(0,1)$ such that 
    \begin{align*}
        \|V^* - V^{\pi_k}\|_{\infty} \leq (1-(1-\kappa)\beta)^{k-1} \left[ \frac{\|V^*-V^{\pi_0}\|_\infty + 2\tau\nu_{\max}}{1-\kappa} \right], \quad \forall\, k \geq 1.
    \end{align*}
\end{theorem}
The above result shows that Algorithm~\ref{alg:DSPI_SSP} achieves global geometric convergence. The convergence rate is similar to that in discounted tabular setting, except for the factor $1/(1-\kappa)$ which arises due to the Bellman optimality operator not being a contraction mapping with respect to the $\ell_\infty$ norm. Moreover, it leads to the following iteration complexity. 
\begin{corollary}\label{coro:SC_DSPI_SSP}
    Let $\beta_0=1$, $\beta_k=1/2$ for $k\geq1$, and $\tau=1/\nu_{\max}$ in Algorithm \ref{alg:DSPI_SSP}. Then for any $\epsilon>0$, to achieve $\|V^* - V^{\pi_k}\|_{\infty} \leq \epsilon$, the iteration complexity is $\mathcal{O}((1-\kappa)^{-1}\log(\epsilon^{-1}(1-\kappa)^{-1})$. 
\end{corollary}

\subsection{Finite-Time Analysis of NPG for Stochastic Shortest Path Problems}
As a special case of DSPI, NPG also enjoys global geometric convergence, as stated in the following theorem. 
\begin{theorem}\label{thm:NPG_SSP}
    Consider $\{\pi_k\}$ generated by Algorithm~\ref{alg:PDA_SSP}.  
    When choosing $\alpha_0=\log(m)$ and $\alpha_{k} = \beta\alpha_0/(1-\beta)^{k}$ for all $k\geq 1$ and for any $\beta\in (0,1)$ , we have 
        \begin{align*}
            \|V^* - V^{\pi_k}\|_\infty
            \le (1 - (1 - \kappa)\beta)^{k-1}
            \left[ \frac{\|V^*-V^{\pi_0}\|_\infty + 2}{1-\kappa}\right],
            \quad \forall\, k \geq 1.
        \end{align*}
\end{theorem}
The proof of Theorem~\ref{thm:NPG_SSP} follows from combining Proposition~\ref{prop:equivalence_SSP} and Theorem~\ref{thm:DSPI-SSP}. To apply Proposition~\ref{prop:equivalence_SSP}, we first verify that $\{\alpha_k\}$ as chosen in Theorem~\ref{thm:NPG_SSP} corresponds to $\{\beta_k\}$ as chosen in Theorem~\ref{thm:DSPI-SSP}. We then substitute the specific choices $\nu(\cdot) = h(\cdot)$ and $\tau = 1/\log(m)$. These steps follow the arguments in Appendix~\ref{pf:thm:NPG} and hence are omitted. We now state the iteration complexity of NPG.
\begin{corollary}\label{coro:NPG_SSP}
    Let $\alpha_0=\log(m)$ and $\alpha_k=\log(m)\cdot 2^{k-1}$ for all $k\geq 1$ in Algorithm \ref{alg:PDA_SSP}. Then for any $\epsilon>0$, to achieve $\|V^*-V^{\pi_k}\|_\infty\le \epsilon$ the iteration complexity is $\mathcal{O}((1-\kappa)^{-1}\log(\epsilon^{-1}(1-\kappa)^{-1})$. 
\end{corollary}
As discussed, the analysis of NPG for SSP problems has received relatively limited attention compared to the discounted setting. To the best of our knowledge, the only existing analysis of NPG in the SSP setting is due to \cite{lee2025policy}, where it is studied as part of a broader framework of policy gradient methods for undiscounted total-reward MDPs. For NPG, \cite{lee2025policy} establish an $\mathcal{O}(1/k)$ convergence rate with a constant stepsize and a geometric convergence rate with adaptive stepsizes. Their analysis is based on the mirror descent framework of \cite{xiao2022convergence}, resulting in convergence rates that depend on a distribution mismatch coefficient; moreover, the adaptive stepsizes depend on the optimal policy. In our setting, Theorem~\ref{thm:NPG_SSP} establishes distribution-free geometric convergence using a simple geometrically increasing stepsize. As mentioned earlier, policy gradient methods for SSP problems have also been studied by \cite{chen2022policy}, but in an episodic setting and from a regret perspective; therefore, their results are not directly comparable to ours. 

\subsection{Proof of Theorem \ref{thm:DSPI-SSP}}\label{sec:pf:DSPI-SSP}
The proof follows the same line of argument as before, establishing monotonic policy improvement and a contractive recursion. The following lemma formalizes the monotonicity of the DSPI policy updates and is proved in Appendix~\ref{pf:le:monotone_SSP}. 
\begin{lemma}\label{le:monotone_SSP}
    Under Assumption \ref{as:proper-all}, the policies $\pi_k$ generated by Algorithm \ref{alg:DSPI_SSP} satisfy $Q^{\pi_{k+1}} \geq Q^{\pi_k}$ for all $k \geq 0$. 
\end{lemma}
\begin{remark}\label{rem:monotone_SSP_weak}
    Although we are assuming that all stationary policies are proper, monotonic policy improvement can be established under the assumption that (1) $\pi_0$ is proper, and (2) for every improper policy $\pi$, there exists $(s,a)\in\mathcal{S}\times\mathcal{A}$ such that $Q^{\pi}(s,a) = -\infty$. We present the proof in Appendix~\ref{ap:monotone_SSP_weak}. Such an assumption is similar to the assumptions formulated in \cite{bertsekas1991analysis}, where the authors require (1) the existence of at least one proper policy, and (2) that every improper policy must lead to $Q^{\pi}(s,a) = -\infty$ for some $(s,a)$. However, as discussed, existing analyses require all policies being proper in order to establish convergence rate. The condition that $\pi_0$ be proper is in view of the initialization in Algorithm \ref{alg:DSPI_SSP}. 
\end{remark}

Following our approach thus far, the next step is to establish a contractive recursion, which relies on the contraction property of the Bellman optimality operator. However, unlike the discounted setting, the Bellman optimality operator $\widetilde{\mathcal{H}}$ in the SSP setting is not a $\ell_\infty$ norm contraction in general; rather, it is a contraction with respect to a weighted $\ell_\infty$ norm \cite{bertsekas1996neuro}. We formally state this property in the following lemma. 
\begin{lemma}[Proposition 2.2 of \cite{bertsekas1996neuro}]\label{le:weighted_sup_norm}
    Under Assumption~\ref{as:proper-all}, there exists $\xi \in \mathbb{R}^{mn}$ such that $\xi > \mathbf{1}$ and $\widetilde{\mathcal{H}}$ is a $\kappa$-contraction with respect to $\|\cdot\|_{\xi}$, where 
    \begin{align*}
        \kappa = \max_{1\leq i\leq mn}\frac{\xi(i)-1}{\xi(i)} \in (0,1) \quad \text{and} \quad \|Q\|_{\xi} \coloneqq \max_{1\leq i \leq mn}\frac{|Q(i)|}{\xi(i)},\;\quad \forall\, Q\in\mathbb{R}^{mn}.
    \end{align*}
\end{lemma}
The following result shows the equivalence between the weighted $\ell_\infty$ norm $\|\cdot\|_{\xi}$ and the unweighted $\ell_\infty$ norm $\|\cdot\|_{\infty}$, which we prove in Appendix~\ref{pf:coro:norm-equiv}. 
\begin{corollary}\label{coro:norm-equiv}
    It holds for any $Q\in\mathbb{R}^{mn}$ that $(1-\kappa)\|Q\|_{\infty} \leq \|Q\|_{\xi} \leq \|Q\|_{\infty}$. 
\end{corollary}
For the remaining part of the proof, we denote $\kappa' = \min\left\{1,\frac{\kappa}{1-\kappa}\right\}$. Now, we present below the contractive recursion with respect to the weighted $\ell_\infty$ norm $\|\cdot\|_\xi$. 
\begin{lemma}\label{le:contraction_recursion_SSP}
   Under Assumption \ref{as:proper-all}, Algorithm \ref{alg:DSPI_SSP} satisfies  
    \begin{align*}
        \|Q^{\ast} - \bar{Q}_{k+1}\|_\xi \leq (1-(1-\kappa)\beta_k) \|Q^* - \bar{Q}_{k}\|_{\xi} + \kappa' \beta_k \eta_{k-1} \nu_{\max},\quad\forall\, k \geq 1.
    \end{align*}
\end{lemma}
Recursively applying the inequality in Lemma~\ref{le:contraction_recursion_SSP} gives 
\begin{align}\label{eq:barQ_gap_SSP}
    \|Q^{\ast} - \bar{Q}_{k+1}\|_{\xi} \leq \prod_{j=1}^{k}(1-(1-\kappa)\beta_j)\|Q^{\ast} - \bar{Q}_{1}\|_{\xi} + \kappa' \nu_{\max}\sum_{i=1}^k \beta_{i}\eta_{i-1} \prod_{j=i+1}^{k}(1-(1-\kappa)\beta_{j}).
\end{align}
To translate the above bound to one in terms of the value function, we apply the following lemma which is proved in Appendix~\ref{pf:le:translation_SSP}. 
\begin{lemma}\label{le:translation_SSP}
    Let $\Tilde{\xi}(s) = \max_{a\in\mathcal{A}}\xi(s,a)$. Then, under Assumption~\ref{as:proper-all}, Algorithm~\ref{alg:DSPI_SSP} satisfies 
    \begin{enumerate}[(1)]
        \item $\|V^* - V^{\pi_k}\|_{\Tilde{\xi}} \leq \|Q^* - \bar{Q}_{k}\|_{\infty} + \eta_{k-1} \nu_{\max}$ for all $k \geq 1$;
        \item $\|Q^*-\bar{Q}_1\|_\infty \leq \|V^*-V^{\pi_0}\|_\infty$.
    \end{enumerate}
\end{lemma}
For $k\geq 1$, combining Lemma~\ref{le:translation_SSP} with the bound in \eqref{eq:barQ_gap_SSP} yields 
\begin{align}\label{eq:final_recursion_SSP}
    \|V^* - V^{\pi_k}\|_{\Tilde{\xi}} \leq\; & \underbrace{\prod_{j=1}^{k-1}(1-(1-\kappa)\beta_j) \|V^*-V^{\pi_0}\|_\infty}_{\eqqcolon T_1} \nonumber \\
    &\; + \underbrace{\kappa' \nu_{\max}\sum_{i=1}^{k-1} \beta_{i}\eta_{i-1} \prod_{j=i+1}^{k-1}(1-(1-\kappa)\beta_{j})}_{\eqqcolon T_2} + \underbrace{\eta_{k-1} \nu_{\max}}_{\eqqcolon T_3}.
\end{align}
Recall that $\eta_{k-1} \coloneqq \tau \prod_{j=1}^{k-1}(1-\beta_j)$. Thus, when choosing $\beta_k \equiv \beta$, we get 
\begin{align*}
    T_1 = (1-(1-\kappa)\beta)^{k-1} \|V^*-V^{\pi_0}\|_\infty \quad \text{and} \quad T_3 = \tau (1-\beta)^{k-1} \nu_{\max}.
\end{align*}
The term $T_2$ can be simplified as follows
\begin{align*}
    & \kappa' \tau \nu_{\max} \beta \sum_{i=1}^{k-1} (1-\beta)^{i-1} (1-(1-\kappa)\beta)^{k-i-1} \\
    &\;= \kappa' \tau \nu_{\max} \beta (1-(1-\kappa)\beta)^{k-2} \sum_{i=1}^{k-1} \frac{(1-\beta)^{i-1}}{(1-(1-\kappa)\beta)^{i-1}} \\  
    &\;= \kappa' \tau \nu_{\max} \beta (1-(1-\kappa)\beta)^{k-2} \frac{1 - \left(\frac{(1-\beta)}{(1-(1-\kappa)\beta)}\right)^{k-1}}{1 - \frac{(1-\beta)}{(1-(1-\kappa)\beta)}} \\
    &\;= \kappa' \tau \nu_{\max} \beta \frac{\left[ (1-(1-\kappa)\beta)^{k-1} - (1-\beta)^{k-1} \right]}{(1-\beta+\beta \kappa) - (1-\beta)}  \\
    &\; =  \frac{\kappa'}{\kappa}\tau \nu_{\max} \left[ (1-(1-\kappa)\beta)^{k-1} - (1-\beta)^{k-1} \right].
\end{align*}
Now, observe that $\frac{\kappa'}{\kappa} = \min\left\{1,\frac{\kappa}{1-\kappa}\right\}\frac{1}{\kappa} = \min\left\{\frac{1}{\kappa},\frac{1}{1-\kappa}\right\} \in (1,2)$ for $\kappa \in (0,1)$. 
Consequently, the bound on the value function gap simplifies as
\begin{align*}
    \|V^* - V^{\pi_k}\|_{\Tilde{\xi}} \leq \; &(1-(1-\kappa)\beta)^{k-1}\|V^*-V^{\pi_0}\|_\infty \\
    &\; +\frac{\kappa'}{\kappa}\tau \nu_{\max} \left[ (1-(1-\kappa)\beta)^{k-1} - (1-\beta)^{k-1} \right] + \tau\nu_{\max} (1-\beta)^{k-1}  \\
    =\; & (1-(1-\kappa)\beta)^{k-1}\|V^*-V^{\pi_0}\|_\infty + \frac{\kappa'}{\kappa}\tau\nu_{\max}(1-(1-\kappa)\beta)^{k-1} \\
    &\;+ \nu_{\max} \left(1-\frac{\kappa'}{\kappa}\right) (1-\beta)^{k-1} \\
    \leq\; & (1-(1-\kappa)\beta)^{k-1}\left[ \|V^*-V^{\pi_0}\|_\infty + 2\tau\nu_{\max} \right].
\end{align*}
where the last inequality is due to $\kappa'/\kappa \in (1,2)$. Finally, by the equivalence of norms $\|\cdot\|_{\Tilde{\xi}}$ and $\|\cdot\|_{\infty}$, we get 
\begin{align*}
    \|V^* - V^{\pi_k}\|_{\infty} \leq (1-(1-\kappa)\beta)^{k-1} \left[ \frac{\|V^*-V^{\pi_0}\|_\infty + 2\tau\nu_{\max}}{1-\kappa} \right].
\end{align*}

\subsection{Proof of Technical Lemmas in Section \ref{sec:pf:DSPI-SSP}}
\subsubsection{Proof of Lemma~\ref{le:monotone_SSP}}\label{pf:le:monotone_SSP}
We take the same approach as in our previous results showing monotonic policy improvement. It suffices to show that $Q^{\pi_k} \leq \widetilde{\mathcal{H}}^{\pi_{k+1}}(Q^{\pi_k})$ for every $k \geq 0$. Then, by monotonicity of $\widetilde{\mathcal{H}}^{\pi_{k+1}}$, we will have $Q^{\pi_k} \leq \widetilde{\mathcal{H}}^{\pi_{k+1}}(Q^{\pi_k}) \leq [\widetilde{\mathcal{H}}^{\pi_{k+1}}]^{2}(Q^{\pi_k})\leq\ldots\leq [\widetilde{\mathcal{H}}^{\pi_{k+1}}]^{n}(Q^{\pi_k})$ for every $n \in \mathbb{N}$. Under Assumption \ref{as:proper-all}, $\pi_{k+1}$ is a proper policy and hence letting $n\to\infty$ yields $Q^{\pi_{k+1}} = \lim_{n\to\infty}[\widetilde{\mathcal{H}}^{\pi_{k+1}}]^{n}(Q^{\pi_k})$ which implies $Q^{\pi_k} \leq Q^{\pi_{k+1}}$. Again, we prove the claim in two steps: $k = 0$ and $k \geq 1$. 

\begin{enumerate}[(1)]
    \item For $k = 0$, $\beta_0 = 1$ and $\bar{Q}_1 = Q^{\pi_0}$. By properness of $\pi_0$, we have 
    \begin{align*}
        Q^{\pi_0} & = \widetilde{\mathcal{H}}^{\pi_0}(Q^{\pi_0}) \\
        & = \widetilde{\mathcal{H}}^{\pi_0}(\bar{Q}_1) \\
        & = \widetilde{\mathcal{H}}^{\pi_0}_{\eta_0}(\bar{Q}_1) - \eta_0g(\pi_0) \\
        & \leq \widetilde{\mathcal{H}}_{\eta_0}(\bar{Q}_1) - \eta_0g(\pi_0) \tag{$\widetilde{\mathcal{H}}^{\pi}_{\eta} \leq \widetilde{\mathcal{H}}_{\eta}$} \\
        & = \widetilde{\mathcal{H}}^{\pi_1}_{\eta_0}(\bar{Q}_1) - \eta_0g(\pi_0) \tag{Algorithm~\ref{alg:DSPI_SSP}, Line~4} \\
        & = \widetilde{\mathcal{H}}^{\pi_1}(\bar{Q}_1) + \eta_0g(\pi_1) - \eta_0 g(\pi_0) \\
        & \leq \widetilde{\mathcal{H}}^{\pi_1}(\bar{Q}_1).
    \end{align*}
    where the last inequality follows from the initialization $\pi_0(s) \in \argmax_{\mu\in\Delta(\mathcal{A})}\nu(\mu)$. 
    \item For $k \geq 1$, since $\pi_k$ is proper, we proceed as follows:
    \begin{align*}
        Q^{\pi_k} & = \widetilde{\mathcal{H}}^{\pi_k}(Q^{\pi_k}) \\
        & = \widetilde{\mathcal{H}}^{\pi_k}\left(\frac{\bar{Q}_{k+1} - (1-\beta_k)\bar{Q}_{k}}{\beta_k}\right) \tag{Algorithm~\ref{alg:DSPI_SSP}, Line~3} \\
        & = \widetilde{\mathcal{H}}^{\pi_k}_{\eta_{k-1}}\left(\frac{\bar{Q}_{k+1} - (1-\beta_k)\bar{Q}_{k}}{\beta_k}\right) - \eta_{k-1} g(\pi_k) \\
        & = \frac{1}{\beta_k} \widetilde{\mathcal{H}}^{\pi_k}_{\eta_{k-1}}(\bar{Q}_{k+1}) - \left(\frac{1}{\beta_k} - 1\right)\widetilde{\mathcal{H}}^{\pi_k}_{\eta_{k-1}}(\bar{Q}_k) - \eta_{k-1} g(\pi_k) \tag{$\widetilde{\mathcal{H}}^{\pi_k}_{\eta_{k-1}}$ being affine} \\
        & = \frac{1}{\beta_k} \widetilde{\mathcal{H}}^{\pi_k}_{\eta_{k}}(\bar{Q}_{k+1}) - \left(\frac{1}{\beta_k} - 1\right)\widetilde{\mathcal{H}}^{\pi_k}_{\eta_{k-1}}(\bar{Q}_k) + \left[\frac{(1-\beta_k)\eta_{k-1} - \eta_{k}}{\beta_k}\right] g(\pi_k)  \\
        & = \frac{1}{\beta_k} \widetilde{\mathcal{H}}^{\pi_k}_{\eta_{k}}(\bar{Q}_{k+1}) - \left(\frac{1}{\beta_k} - 1\right)\widetilde{\mathcal{H}}^{\pi_k}_{\eta_{k-1}}(\bar{Q}_k),
    \end{align*}
    where the last step is due to $\eta_{k}=(1-\beta_k)\eta_{k-1}$. As per the update rule in Line 4 of Algorithm~\ref{alg:DSPI_SSP}, $\widetilde{\mathcal{H}}_{\eta_k}^{\pi_{k+1}}(\bar{Q}_{k+1})=\widetilde{\mathcal{H}}_{\eta_k}(\bar{Q}_{k+1})$ for $k \geq 0$, and we thus get 
    \begin{align*}
        & \widetilde{\mathcal{H}}^{\pi_k}_{\eta_{k}}(\bar{Q}_{k+1}) \leq \widetilde{\mathcal{H}}_{\eta_{k}}(\bar{Q}_{k+1}) = \widetilde{\mathcal{H}}^{\pi_{k+1}}_{\eta_{k}}(\bar{Q}_{k+1});\\
        & \widetilde{\mathcal{H}}^{\pi_k}_{\eta_{k-1}}(\bar{Q}_k) = \widetilde{\mathcal{H}}_{\eta_{k-1}}(\bar{Q}_k) \geq \widetilde{\mathcal{H}}^{\pi_{k+1}}_{\eta_{k-1}}(\bar{Q}_k).
    \end{align*}
    Applying these in the bound on $Q^{\pi_k}$ yields  
    \begin{align*}
        Q^{\pi_k} \leq\; & \frac{1}{\beta_k} \widetilde{\mathcal{H}}^{\pi_{k+1}}_{\eta_{k}}(\bar{Q}_{k+1}) - \left(\frac{1}{\beta_k} - 1\right) \widetilde{\mathcal{H}}^{\pi_{k+1}}_{\eta_{k-1}}(\bar{Q}_k) \\
        =\; & \frac{1}{\beta_k} \widetilde{\mathcal{H}}^{\pi_{k+1}}(\bar{Q}_{k+1}) - \left(\frac{1}{\beta_k} - 1\right) \widetilde{\mathcal{H}}^{\pi_{k+1}}(\bar{Q}_k) + \left(\frac{\eta_k-(1-\beta_k)\eta_{k-1}}{\beta_k}\right) g(\pi_{k+1}) \\
        =\; & \frac{1}{\beta_k} \widetilde{\mathcal{H}}^{\pi_{k+1}}(\bar{Q}_{k+1}) - \left(\frac{1}{\beta_k} - 1\right) \widetilde{\mathcal{H}}^{\pi_{k+1}}(\bar{Q}_k) \tag{$\eta_k = (1-\beta_k)\eta_{k-1}$} \\
        =\; & \widetilde{\mathcal{H}}^{\pi_{k+1}}(Q^{\pi_k}),
    \end{align*}
    where the last equality is due to $\widetilde{\mathcal{H}}^{\pi_{k+1}}$ being affine and Line 3 of Algorithm~\ref{alg:DSPI_SSP}. 
\end{enumerate}

\subsection{Proof of Corollary \ref{coro:norm-equiv}}\label{pf:coro:norm-equiv}
Observe that since $\xi(s,a) > 1$ for all $(s,a)\in\mathcal{S}\times\mathcal{A}$, 
    \begin{align*}
        \|Q^{\ast} - Q^{\pi_0}\|_{\xi} = \max_{s,a}\frac{|Q^{\ast}(s,a) - Q^{\pi_0}(s,a)|}{\xi(s,a)} < \max_{s,a}|Q^{\ast}(s,a) - Q^{\pi_0}(s,a)| = \|Q^{\ast} - Q^{\pi_0}\|_{\infty}.
    \end{align*}
    Further, from the proof of existence of $\xi$ and $\kappa$ (\cite[Proposition 2.2]{bertsekas1996neuro}), we have 
    \begin{align*}
        \xi(s,a) - 1 \leq \kappa \xi(s,a),\quad \forall\,(s,a)\in\mathcal{S}\times\mathcal{A},
    \end{align*}
    which implies 
    \begin{align*}
        (1-\kappa) \xi(s,a) \leq 1, \text{ or equivalently, } \xi(s,a) \leq \frac{1}{1-\kappa},\;\quad \forall\,(s,a)\in\mathcal{S}\times\mathcal{A}.
    \end{align*}
    It then follows that 
    \begin{align*}
        \|Q\|_{\xi} = \max_{s,a} \frac{|Q(s,a)|}{\xi(s,a)} \geq (1-\kappa)\max_{s,a}|Q(s,a)| = (1-\kappa) \|Q\|_{\infty}.
    \end{align*}

\subsubsection{Proof of Lemma~\ref{le:contraction_recursion_SSP}}\label{pf:le:contraction_recursion_SSP}
For any $k \geq 0$, we have by Line 3 of Algorithm~\ref{alg:DSPI_SSP} that 
\begin{align*}
    Q^{\ast} - \bar{Q}_{k+1} & = Q^\ast - [(1-\beta_k)\bar{Q}_{k} + \beta_k Q^{\pi_k}] \\
    & = Q^\ast - \bar{Q}_{k} + \beta_k \bar{Q}_{k} - \beta_k Q^{\pi_k} \\
    & = Q^{\ast} - \bar{Q}_{k} + \beta_k \bar{Q}_{k} - \beta_{k} Q^{\ast} + \beta_{k} Q^{\ast} - \beta_{k} Q^{\pi_k} \\
    & = (1-\beta_k)(Q^{\ast} - \bar{Q}_{k}) + \beta_k (Q^{\ast} - Q^{\pi_k}). 
\end{align*}
As argued in the proof of Lemma~\ref{le:contraction_recursion}, since $\bar{Q}_k$ is a weighted average of a monotonic sequence, we have $\bar{Q}_k \le Q^{\pi_k}$ for all $k \ge 0$. From this, we can deduce that 
\begin{align*}
    Q^{\ast} - Q^{\pi_k} & = \widetilde{\mathcal{H}}(Q^*) - \widetilde{\mathcal{H}}^{\pi_k}(Q^{\pi_k}) \\
    & \leq \widetilde{\mathcal{H}}(Q^*) - \widetilde{\mathcal{H}}^{\pi_k}(\bar{Q}_{k}) \tag{$\widetilde{\mathcal{H}}^{\pi_k}$ being monotonic} \\
    & = \widetilde{\mathcal{H}}(Q^*) - \widetilde{\mathcal{H}}^{\pi_k}_{\eta_{k-1}}(\bar{Q}_{k}) + \eta_{k-1}g(\pi_{k}) \\
    & = \widetilde{\mathcal{H}}(Q^*) - \widetilde{\mathcal{H}}_{\eta_{k-1}}(\bar{Q}_{k}) + \eta_{k-1}g(\pi_{k}) \tag{Algorithm \ref{alg:DSPI_SSP}, Line 4} \\
    & \leq \widetilde{\mathcal{H}}(Q^*) - \widetilde{\mathcal{H}}(\bar{Q}_{k}) + \eta_{k-1}g(\pi_{k}) \tag{$\widetilde{\mathcal{H}}(Q) \leq \widetilde{\mathcal{H}}_{\eta}$}.
\end{align*}
Under Assumption~\ref{as:proper-all}, we have that $g(\pi) \leq \kappa' \nu_{\max} \mathbf{1}$ for any stationary policy $\pi$, which we show in Appendix~\ref{ap:bound_g}. To proceed, for any $(s,a)\in\mathcal{S}\times\mathcal{A}$, we have 
\begin{align*}
    \frac{1}{\xi(s,a)}\left|Q^{\ast}(s,a) - Q^{\pi_k}(s,a)\right| & = \frac{1}{\xi(s,a)} [Q^{\ast}(s,a) - Q^{\pi_k}(s,a)]  \\
    & \leq \frac{1}{\xi(s,a)}\left\{ [\widetilde{\mathcal{H}}(Q^*)](s,a) - [\widetilde{\mathcal{H}}(\bar{Q}_{k})](s,a) \right\} + \frac{\eta_{k-1}}{\xi(s,a)}[g(\pi_k)](s,a) \\
    & \leq \frac{1}{\xi(s,a)}\left| [\widetilde{\mathcal{H}}(Q^*)](s,a) - [\widetilde{\mathcal{H}}(\bar{Q}_{k})](s,a) \right| + \kappa' \eta_{k-1} \nu_{\max} \tag{$\xi \geq \mathbf{1}$} \\
    & = \|\widetilde{\mathcal{H}}(Q^*) - \widetilde{\mathcal{H}}(\bar{Q}_{k})\|_{\xi} + \kappa' \eta_{k-1} \nu_{\max} \\
    & \leq \kappa \|Q^* - \bar{Q}_{k}\|_{\xi} + \kappa' \eta_{k-1} \nu_{\max}, 
\end{align*}
where the last inequality follows due to Lemma~\ref{le:weighted_sup_norm}. Since the above holds for every $(s,a)\in\mathcal{S}\times\mathcal{A}$, we obtain
\begin{align*}
    \|Q^{\ast} - Q^{\pi_k}\|_{\xi} \leq \kappa \|Q^* - \bar{Q}_{k}\|_{\xi} + \kappa' \eta_{k-1} \nu_{\max}.
\end{align*}
Now, we bound $\|Q^{\ast} - \bar{Q}_{k+1}\|_\xi$ as follows:
\begin{align*}
    \|Q^{\ast} - \bar{Q}_{k+1}\|_\xi \leq\; & (1-\beta_k)\|Q^{\ast} - \bar{Q}_{k}\|_{\xi} + \beta_k \|Q^{\ast} - Q^{\pi_k}\|_{\xi} \\
    \leq\; & (1-(1-\kappa)\beta_k) \|Q^* - \bar{Q}_{k}\|_{\xi} + \kappa' \beta_k \eta_{k-1} \nu_{\max}.
\end{align*}

\subsubsection{Bound on \texorpdfstring{$g$}{TEXT}}\label{ap:bound_g}
By definition of the operator $g$, we have for every $(s,a)\in\mathcal{S}\times\mathcal{A}$ that 
    \begin{align*}
        [g(\pi)](s,a) = \sum_{s'\in\mathcal{S}} p(s'\mid s, a) \nu(\pi(s')),\quad s\in\mathcal{S},\;a\in\mathcal{A}.
    \end{align*}
    which satisfies $\widetilde{\mathcal{H}}^{\pi}_{\eta}(Q) = \widetilde{\mathcal{H}}^{\pi}(Q) + \eta g(\pi)$ for any $\eta \geq 0$ and any policy $\pi$. 
    Now, by definition of $\widetilde{\mathcal{H}}$, we see that 
    \begin{align*}
        [\widetilde{\mathcal{H}}(\mathbf{0})](s,a) = \mathcal{R}(s,a) \text{ and } [\widetilde{\mathcal{H}}(\mathbf{1})](s,a) = \mathcal{R}(s,a) + \sum_{s'\in\mathcal{S}}p(s'|s,a).
    \end{align*}
    According to Lemma \ref{le:weighted_sup_norm}, $\widetilde{\mathcal{H}}$ is a contraction mapping with respect to $\|\cdot\|_{\xi}$ and we thus have 
    \begin{align*}
        \|\widetilde{\mathcal{H}}(\mathbf{1}) - \widetilde{\mathcal{H}}(\mathbf{0})\|_{\xi} = \max_{s,a}\frac{\sum_{s'\in\mathcal{S}} p(s'|s,a)}{\xi(s,a)} \leq \kappa \|\mathbf{1}\|_{\xi} \leq \kappa \|\mathbf{1}\|_{\infty} = \kappa.
    \end{align*}
    Since $1/\xi(s,a) \geq 1/[\max_{s,a}\xi(s,a)]$ for every $(s,a)$, we get 
    \begin{align*}
        \frac{1}{\max_{s,a}\xi(s,a)} \max_{s,a}\sum_{s'\in\mathcal{S}} p(s'|s,a) & \leq \max_{s,a}\frac{\sum_{s'\in\mathcal{S}} p(s'|s,a)}{\xi(s,a)} \leq \kappa \\
        \Rightarrow \max_{s,a}\sum_{s'\in\mathcal{S}} p(s'|s,a) & \leq \kappa \max_{s,a}\xi(s,a) \leq \frac{\kappa}{1-\kappa}. 
    \end{align*}
    Note that $\kappa/(1-\kappa) > 1$ only if $\kappa > 1/2$. Since $p(s'|s,a) \in [0,1)$, $\sum_{s'\in\mathcal{S}} p(s'|s,a) \leq 1$, 
    and we get
    \begin{align*}
        [g(\pi)](s,a) = \sum_{s'\in\mathcal{S}} p(s'\mid s, a) \nu(\pi(s')) \leq \nu_{\max}\sum_{s'\in\mathcal{S}} p(s'\mid s, a) \leq \nu_{\max}\min\left\{1,\frac{\kappa}{1-\kappa}\right\} = \nu_{\max} \kappa'.
    \end{align*}
     for every $(s,a)\in\mathcal{S}\times\mathcal{A}$, and hence $g(\pi) \leq \kappa' \nu_{\max} \mathbf{1}$. 

\subsubsection{Proof of Lemma~\ref{le:translation_SSP}}\label{pf:le:translation_SSP}
\begin{enumerate}[(1)]
    \item For any $s\in\mathcal{S}$, we have that 
    \begin{align*}
        &\frac{1}{\Tilde{\xi}(s)}[V^*(s) - V^{\pi_k}(s)] \\
        & \leq \frac{1}{\Tilde{\xi}(s)}\left[ \max_{a\in\mathcal{A}}Q^*(s,a) - \pi_k(s)^\top \bar{Q}_k(s) \right] \tag{$Q^{\pi_k} \geq \bar{Q}_k$} \\
        & = \frac{1}{\Tilde{\xi}(s)}\left[ \max_{a\in\mathcal{A}}Q^*(s,a) - \max_{\mu\in\Delta(\mathcal{A})} \left\{\mu^\top \bar{Q}_{k}(s) + \eta_{k-1}\nu(\mu)\right\} \right] + \frac{1}{\Tilde{\xi}(s)} \eta_{k-1} \nu(\pi_{k}(s)) \tag{Algorithm \ref{alg:DSPI_SSP}, Line 4}\\
        & \leq \frac{1}{\Tilde{\xi}(s)}\left[ \max_{a\in\mathcal{A}}Q^*(s,a) - \max_{a\in\mathcal{A}} \bar{Q}_{k}(s,a) \right] + \frac{1}{\Tilde{\xi}(s)} \eta_{k-1} \nu(\pi_{k}(s)) \\
        & \leq \frac{1}{\Tilde{\xi}(s)} \max_{a\in\mathcal{A}}\left| Q^*(s,a) - \bar{Q}_{k}(s,a) \right| + \frac{1}{\Tilde{\xi}(s)} \eta_{k-1} \nu(\pi_{k}(s)) \\
        & \leq \max_{a\in\mathcal{A}} \frac{\left| Q^*(s,a) - \bar{Q}_{k}(s,a) \right|}{\xi(s,a)} + \frac{1}{\Tilde{\xi}(s)} \eta_{k-1} \nu(\pi_{k}(s)) \tag{$\Tilde{\xi}(s) \geq \xi(s,a)$, $\forall a\in\mathcal{A}$} \\
        & \leq \|Q^* - \bar{Q}_{k}\|_{\xi} + \frac{1}{\Tilde{\xi}(s)} \eta_{k-1} \nu(\pi_{k}(s)) \\
        & \leq \|Q^* - \bar{Q}_{k}\|_{\infty} + \eta_{k-1} \nu_{\max},
    \end{align*}
    where the last inequality follows from the equivalence of norms $\|\cdot\|_\xi$ and $\|\cdot\|_\infty$, and $\xi \geq \mathbf{1}$. 
    \item Since $\beta_0=1$, we have $\bar{Q}_1=Q^{\pi_0}$. Now, for any $(s,a)$, observe that
    \begin{align*}
        Q^*(s,a)-Q^{\pi_0}(s,a)= \sum_{s'\in\mathcal{S}}p(s'\mid s,a) (V^*(s')- V^{\pi_0}(s'))
        \leq \|V^*-V^{\pi_0}\|_\infty,
    \end{align*}
    which implies $\|Q^*-Q^{\pi_0}\|_\infty \leq \|V^*-V^{\pi_0}\|_\infty$. 
\end{enumerate}

\subsubsection{Monotonic Improvement Under A Weaker Assumption}\label{ap:monotone_SSP_weak}
We first recall the assumptions from Remark~\ref{rem:monotone_SSP_weak}. 
\begin{assumption}\label{as:proper-initial}
    \begin{enumerate}[(1)]
        \item The policy $\pi_0$ is proper.
        \item For every improper policy $\pi$, there exists $(s,a)\in\mathcal{S}\times\mathcal{A}$ such that $Q^{\pi}(s,a) = -\infty$. 
    \end{enumerate}
\end{assumption}
Under this assumption, we first establish the following result, whose proof follows arguments similar to those of \cite[Proposition 2.4]{bertsekas1991analysis}. 
\begin{lemma}\label{le:proper}
    Let Assumption \ref{as:proper-initial} hold. Let $\pi$ be a proper policy and $\pi'$ be any policy such that $Q^{\pi} \leq \widetilde{\mathcal{H}}^{\pi'}(Q^{\pi})$. Then $\pi'$ is also a proper policy. 
\end{lemma}
\begin{proof}
    We prove by contradiction. Suppose $\pi'$ is not proper. 
    First, observe that by monotonicity of $\widetilde{\mathcal{H}}^{\pi'}$, $Q^{\pi} \leq \widetilde{\mathcal{H}}^{\pi'}(Q^{\pi}) \leq [\widetilde{\mathcal{H}}^{\pi'}]^{2}(Q^{\pi}) \leq \cdots \leq [\widetilde{\mathcal{H}}^{\pi'}]^{n}(Q^{\pi})$ for every $n\in \mathbb{N}$. 
    Now, for any $Q \in \mathbb{R}^{mn}$, $\widetilde{\mathcal{H}}^{\pi'}(Q) = \mathcal{R} + \mathcal{P}^{\pi'} Q$, 
    where $\mathcal{P}^{\pi'}$ is a $mn\times mn$ matrix with $\mathcal{P}^{\pi'}((s,a),(s',a')) = p(s'|s,a)\pi'(a'|s')$ for $(s,a),(s',a')\in\mathcal{S}\times\mathcal{A}$. 
    Consequently, $[\widetilde{\mathcal{H}}^{\pi'}]^{n}$ can be written as 
    \begin{align*}
        [\widetilde{\mathcal{H}}^{\pi'}]^{n}(Q) & = \widetilde{\mathcal{H}}^{\pi'}[\widetilde{\mathcal{H}}^{\pi'}]^{n-1}(Q) \\
        & = \mathcal{R} + \mathcal{P}^{\pi'} [\widetilde{\mathcal{H}}^{\pi'}]^{n-1}(Q) \\
        & = \mathcal{R} + \mathcal{P}^{\pi'} [\mathcal{R} + \mathcal{P}^{\pi'} [\widetilde{\mathcal{H}}^{\pi'}]^{n-2}(Q)] \\
        & = \mathcal{R} + \mathcal{P}^{\pi'} \mathcal{R} + (\mathcal{P}^{\pi'})^{2} [\widetilde{\mathcal{H}}^{\pi'}]^{n-2}(Q). 
    \end{align*}
    Continuing in the above way thus leads to 
    \begin{align*}
        [\widetilde{\mathcal{H}}^{\pi'}]^{n}(Q) = \sum_{i=0}^{n-1} (\mathcal{P}^{\pi'})^{i} \mathcal{R} + (\mathcal{P}^{\pi'})^{n}Q.
    \end{align*}
    Also notice that for any $i \in \{1,\ldots,n-1\}$, $(\mathcal{P}^{\pi'})^{i}((s,a),(s',a'))$ is the probability of going from $(s,a)$ to $(s',a')$ in $i$-transitions. Thus, we can write 
    \begin{align*}
        \sum_{i=0}^{n-1} [(\mathcal{P}^{\pi'})^{i} \mathcal{R}](s,a) = \sum_{i=0}^{n-1} \mathbb{E}[\mathcal{R}(S_i,A_i) \mid S_0 = s, A_0 = a]. 
    \end{align*}
    Since, $\pi'$ is improper, by Assumption~\ref{as:proper-initial}(2), there exists $(s,a)\in\mathcal{S}\times\mathcal{A}$ such that 
    \begin{align*}
        \limsup_{n}\sum_{i=0}^{n-1} [(\mathcal{P}^{\pi'})^{i} \mathcal{R}](s,a) = -\infty,
    \end{align*}
    which implies $\limsup_{n}[(\widetilde{\mathcal{H}}^{\pi'})^{n}(Q)](s,a) = -\infty$. Therefore, we have $Q^{\pi}(s,a) = -\infty$, which is a contradiction given $\pi$ is proper. Hence, $\pi'$ cannot be improper. 
\end{proof}

With the above lemma in hand, it is now enough to show that $\pi_k$'s are proper and $Q^{\pi_k} \leq \widetilde{\mathcal{H}}^{\pi_{k+1}}(Q^{\pi_k})$ for every $k\geq 0$. Indeed, this implies that $\pi_{k+1}$ is proper by Lemma~\ref{le:proper}, and $Q^{\pi_{k+1}}$ will be the unique solution to the Bellman equation $\widetilde{\mathcal{H}}^{\pi_{k+1}}(Q) = Q$. Following the same arguments as before, we can apply the monotonicity of $\widetilde{\mathcal{H}}^{\pi_{k+1}}$ to get $Q^{\pi_k} \leq \widetilde{\mathcal{H}}^{\pi_{k+1}}(Q^{\pi_k}) \leq [\widetilde{\mathcal{H}}^{\pi_{k+1}}]^{2}(Q^{\pi_k})\leq\ldots\leq [\widetilde{\mathcal{H}}^{\pi_{k+1}}]^{n}(Q^{\pi_k})$ for every $n \in \mathbb{N}$. Since $Q^{\pi_{k+1}} = \lim_{n\to\infty}[\widetilde{\mathcal{H}}^{\pi_{k+1}}]^{n}(Q^{\pi_k})$, we will obtain $Q^{\pi_{k}} \leq Q^{\pi_{k+1}}$. 

To prove the claim, we proceed as follows. Starting from $\pi_0$, the arguments in Appendix~\ref{pf:le:monotone_SSP} (1) imply that $Q^{\pi_0} \leq \widetilde{\mathcal{H}}^{\pi_1}(Q^{\pi_0})$. Since $\pi_0$ is proper, it follows from Lemma~\ref{le:proper} that $\pi_1$ is also proper. Next, the arguments in Appendix~\ref{pf:le:monotone_SSP} (2) ensure that $Q^{\pi_1} \leq \widetilde{\mathcal{H}}^{\pi_2}(Q^{\pi_1})$ and hence Lemma~\ref{le:proper} implies that $\pi_2$ is proper. Continuing inductively, applying the arguments in Appendix~\ref{pf:le:monotone_SSP} (2) and Lemma~\ref{le:proper}, we can conclude that $\pi_k$'s are proper and $Q^{\pi_k} \leq \mathcal{H}^{\pi_{k+1}}(Q^{\pi_k})$ for every $k\geq 0$.

\end{document}